%% file: main.tex
\documentclass[12pt,a4paper]{report}

\usepackage{thesis}
\usepackage{lmodern}
\usepackage[numbers,sort&compress]{natbib}
\usepackage[hidelinks]{hyperref}
\usepackage[left=2.5cm,top=2.5cm,right=2.5cm,bottom=2.5cm]{geometry}
\usepackage{nomencl}
\makenomenclature

\usepackage{tocloft}

\newcommand{\thesistitle}{Towards Robust Out-of-Distribution Generalization:\\Data Augmentation and Neural Architecture Search Approaches
}
\newcommand{\thesisauthor}{Haoyue Bai}
\newcommand{\programname}{Computer Science and Engineering}
\newcommand{\departmentname}{Department of Computer Science and Engineering}
\newcommand{\thesisdate}{May 2022}

\newcommand{\thesisdegree}{PhD}

\usepackage[pdftex]{graphicx}
\usepackage{amsmath,amsfonts}
\usepackage{amssymb,amsmath,amsthm}
\newtheorem{theorem}{Theorem}

\usepackage{algorithmic}
\usepackage{algorithm}
\usepackage{array}
\usepackage[caption=false,font=normalsize,labelfont=sf,textfont=sf]{subfig}
\usepackage{textcomp}
\usepackage{stfloats}
\usepackage{url}
\usepackage{verbatim}
\usepackage{graphicx}

\usepackage{times}
\usepackage{epsfig}
\usepackage{graphicx}
\usepackage{amsmath}
\usepackage{amssymb}

\usepackage{setspace}

\usepackage{threeparttable}
\usepackage{adjustbox}
\usepackage{multirow}
\usepackage{mathtools}

\usepackage{booktabs}
\usepackage{bm}
\usepackage{makecell}
\usepackage[normalem]{ulem}
\usepackage{comment}

\newcommand{\etal}{\textit{et al. }}

\DeclareMathAlphabet{\mathcal}{OMS}{cmsy}{m}{n}

\usepackage{amsmath}
\usepackage{mathtools}

\newcommand{\ie}{\textit{i}.\textit{e}.} 
\usepackage[switch]{lineno}
 
\newcommand{\name}{DecAug}

\newtheorem{lemma}{Lemma}

\begin{document}

\pagenumbering{roman}
\pagestyle{plain}
\setcounter{page}{1}
\addcontentsline{toc}{chapter}{Title Page}
\include{1_title}

\newpage
\addcontentsline{toc}{chapter}{Acknowledgments}
\include{5_acknowledgement}

\newpage
\addcontentsline{toc}{chapter}{Table of Contents}
\tableofcontents

\newpage
\addcontentsline{toc}{chapter}{List of Figures}
\listoffigures

\newpage
\addcontentsline{toc}{chapter}{List of Tables}
\listoftables

\newpage
\addcontentsline{toc}{chapter}{Abstract}

\begin{center}
{\Large \thesistitle}\\
\vspace{20mm}
by \thesisauthor\\
\departmentname\\
The Hong Kong University of Science and Technology
\end{center}
\vspace{8mm}
\begin{center}
Abstract
\end{center}
\par
\noindent

Deep learning has been demonstrated with tremendous success in recent years. Despite so, its performance in practice often degenerates drastically when encountering out-of-distribution (OoD) data, i.e. training and test data are sampled from different distributions. In this thesis, we study ways toward robust OoD generalization for deep learning, i.e., its performance is not susceptible to distribution shift in the test data.

We first propose a novel and effective approach to disentangle the spurious correlation between features that are not essential for recognition. It employs decomposed feature representation by orthogonalizing the two gradients of losses for category and context branches. Furthermore, we perform gradient-based 
augmentation on context-related features (e.g., styles, backgrounds, or scenes of target objects) to improve the robustness of learned representations. Results show that our approach generalizes well for different distribution shifts.

We then study the problem of strengthening neural architecture search in OoD scenarios.
We propose to optimize the architecture parameters that minimize the validation loss 
on synthetic OoD data, under the condition that corresponding network parameters minimize the training loss. 
Moreover, to obtain a proper validation set, we learn a conditional generator by maximizing their losses computed by different neural architectures.
Results show that our approach effectively discovers robust architectures that perform well for OoD generalization.

\newpage
\pagenumbering{arabic}
\pagestyle{plain}
\setcounter{page}{1}

\chapter{Introduction}

Deep learning has achieved tremendous success in various applications of computer vision~\cite{he2016deep} and natural language processing~\cite{devlin2018bert}, under the implicit assumption that the training and test data are drawn independent and identically distributed (IID) from the same distribution.

While neural networks often exhibit super-human generalization performance on the training distribution, they can be susceptible to minute changes in the test distribution~\cite{recht2019imagenet, szegedy2014intriguing}.
This is problematic because sometimes true underlying data distributions are significantly underrepresented or misrepresented by the limited training data at hand.
In the real world, such mismatches are commonly observed~\cite{koh2020wilds, geirhos2020shortcut, ye2022ood}, and have led to significant performance drops in many deep learning algorithms~\cite{bahng2019rebias, krueger2020outofdistribution, mancini2020towards}.
As a result, the reliability of current learning systems is substantially undermined in critical applications such as medical imaging~\cite{castro2020causality, albadawy2018deep}, autonomous driving~\cite{dai2018dark, volk2019towards, alcorn2019strike, sato2020security, mao2021voxel, mao2021pyramid, michaelis2019benchmarking}, and security systems~\cite{huang2020survey}.

The task of generalizing under such distribution shifts, has been fragmentarily researched in different areas, such as Domain Generalization (DG)~\cite{blanchard2011generalizing, muandet2013domain, wang2021generalizing, zhou2021domain, bai2023improving}, Causal Inference~\cite{pearl2010causal, peters2017elements}, and Stable Learning~\cite{zhang2021deep}.
In the setting of OoD generalization, models usually have access to multiple training datasets of the same task collected in different environments.
The goal of towards robust OoD generalization is to learn from these different but related training environments and then extrapolate to unseen test environments~\cite{arjovsky2020out, shen2021towards}, which means the deep neural networks preserve the super-human generalization performance in the mismatched test distribution. In this thesis, we focus on discussing two approaches for tackling OoD generalization challenges: semantic data augmentation and neural architecture search approaches.
Data and network architectures play pivotal roles in constructing machine learning systems in practice. It should be noticed that our semantic data augmentation approach and neural architecture search approach is diagonal to each other and that can be easily combined.

In the first part of this dissertation, we explore disentangled representation and data augmentation for OoD Generalization. We discuss DecAug, a novel decomposed feature representation and semantic augmentation approach for OoD generalization~\cite{bai2020decaug}. Specifically, DecAug disentangles the category-related and context-related features by orthogonalizing the two gradients (w.r.t. intermediate features) of losses for predicting category and context labels, where category-related features contain causal information of the target object, while context related features cause distribution shifts between training and test data. Furthermore, we perform gradient-based augmentation on context-related features to improve the robustness of the learned representations. Experimental results show that DecAug outperforms other state-of-the-art methods on various OoD datasets, which is among the very few methods that can deal with different types of OoD generalization challenges.

In the second part of the dissertation, we explore robust neural architecture search for OoD generalization~\cite{bai2021ood}.
Recent advances on Out-of-Distribution (OoD) generalization reveal the robustness of deep learning models against different kinds of distribution shifts in real-world applications. However, existing works focus on OoD algorithms, such as invariant risk minimization, domain generalization, or stable learning, without considering the influence of deep model architectures on OoD generalization, which may lead to sub-optimal performance. Neural Architecture Search (NAS) methods search for the architecture based on its performance on the training data, which may result in poor generalization for OoD tasks. In this work, we propose robust Neural Architecture Search for OoD generalization (NAS-OoD), which optimizes the architecture with respect to its performance on the generated OoD data by gradient descent. Specifically, a data generator is learned to synthesize OoD instances by maximizing their losses computed by different neural architectures, while the goal for architecture search is to find the optimal architecture parameters that minimize the synthetic OoD data losses. The data generator and the neural architecture are jointly optimized in an end-to-end manner, and the minimax training process effectively discovers robust architectures that generalize well for different distribution shifts. Extensive experimental results show that NAS-OoD achieves superior performance on various OoD generalization benchmarks with deep models having a much fewer number of parameters. In addition, on a real industry dataset, the proposed NAS-OoD method reduces the error rate by more than 70\% compared with the state-of-the-art method, demonstrating the proposed method’s practicality for real applications.

\clearpage

\chapter{Background and Preliminaries}

\section{Out-of-Distribution Generalization Robustness}

Several main approaches for solving the OoD generalization problem can be identified in the literature, including risk regularization methods, domain generalization, stable learning, data augmentation and disentangled representation.
Currently, there are two main streams of research on OoD generalization algorithms in deep learning: Domain Generalization (DG) and risk regularization methods.
DG has been an active research area for quite some time, dating back to the work of \cite{torralba2011unbiased}.
It has since then sprouted into a number of branches such as invariant representation learning \cite{muandet2013domain, li2018domain, li2018domain2, akuzawa2019adversarial, albuquerque2019adversarial}, meta-learning for DG \cite{li2018learning, balaji2018metareg, dou2019domain, li2019feature}, and data augmentation for DG \cite{volpi2018generalizing, zhang2019unseen, wang2020heterogeneous, wang2020learning, zhou2020deep}.
The seminal work of risk regularization methods, IRM~\cite{arjovsky2019invariant}, aims to find an invariant representation of data from different training environments by adding an invariant risk regularization. It has thereafter inspired several other notable algorithms, including IRM-Games~\cite{ahuja2020invariant}, VREx~\cite{krueger2020outofdistribution} and IGA~\cite{koyama2020out}.
The living benchmark is created by \cite{gulrajani2021in} to facilitate disciplined and reproducible DG research.
After conducting a large-scale hyperparameter search, the performances of fourteen algorithms on seven datasets are reported in the paper.
The authors then arrive at the conclusion that ERM beats most of DG algorithms under the same fair setting.

Risk regularization methods for OoD generalization 
are motivated by the theory of causality and causal Bayesian networks (CBNs), aiming to find an invariant representation of data from different training environments (IRM, \cite{arjovsky2019invariant}). 
To make the model robust to unseen test environments, the invariant risk minimization added a penalty item in the loss function to monitor the optimality of a classifier on different environments. 
IRM-Games~\cite{ahuja2020invariant}borrowed the principle of IRM and provided another way to find the invariant feature.
Risk extrapolation (REx, \cite{krueger2020outofdistribution}) adopts a min-max framework to derive a model that can perform well on the worst linear combination of risks from different environments. 
These methods typically perform well on datasets containing correlation shift, such as Colored MNIST, Rotated MNIST, etc. However, it is demonstrated that they cannot generalize well on datasets containing diversity shift.
\noindent\textit{Domain generalization} methods can be divided into several different strategies to learn invariant predictors:learning invariant features, sharing parameters, meta-learning, or
performing data augmentation\cite{gulrajani2020search}.
Fabio \etal proposed a self-supervised learning method JiGen. It achieves good performance in typical domain generalization datasets, such as PACS \cite{carlucci2019domain} and VLCS\cite{fang2013unbiased}. In JiGen, a subnetwork sharing weights with the main network is used to solve Jigsaw puzzles. This self-supervised learning method helps in improving the generalization of JiGen on unseen domains. Qi \etal adopted meta-learning to learn invariant feature representations across domains \cite{dou2019domain}. 
Fengchun \etal considered a worst-case scenario in model generalization, where there is only one single domain for training~\cite{qiao2020learning}.
InfoDrop~\cite{shi2020informative} proposes a light-weight model-agnostic informative dropout to reduce texture bias and improve robustness.
Zeyi \etal introduces a training scheme representation self-challenging to improve the generalization of CNN to the out-of-domain data~\cite{huang2020self}.
Recently, Mancini \etal proposed the curriculum mixup method for domain generalization, in which data from multiple domains in the training dataset mix together by a curriculum schedule of mixup method \cite{mancini2020towards}. Domain generalization methods have achieved performance gain in generalizing models to unseen domains, typically different image styles. However, in recent OoD research, it has been found that domain adaptation methods with similar design principles can have problems when training distribution is largely different from test distribution \cite{arjovsky2019invariant}.

Stable learning is a recently proposed new concept \cite{Kuang2018}, which focuses on learning a model that can achieve stable performances across different environments. The methodology of stable learning largely inherited from sampler reweighting in causal inference \cite{Kuang2018}. For example, in the convents with batch balancing algorithm (CNBB), instead of viewing all samples in the dataset equally, it first calculates the weights of each sample by a confounder balancing loss which tests whether including or excluding a sample's feature can lead to a significant change in the latent feature representations to reduce the effects of samples that are largely affected by confounders \cite{Kosuke2014,Robins1994}. While these methods can have theoretical guarantees on simplified models, when confounder results in strong spurious correlations, this method may not be able to work well especially in the deep learning paradigm.

\section{Disentangled Representation and Semantic Augmentation}

Disentangled representation 
proposes to decompose the latent factors from the image variants to obtain an understanding of the data~\cite{chen2016infogan,higgins2017beta,ma2019disentangled}.
CSD~\cite{piratla2020efficient} decomposes and learns a common component and a domain-specific component for generalizing to new domains.
It aims to learn representations that separate the explanatory factors of variations behind the data. Such representations are demonstrated to be more resilient to the complex variants and able to bring enhanced generalization ability~\cite{liu2018detach,peng2019domain}.
Moreover, disentangled representations are inherently more interpretable.
How to obtain disentangled representations is still a challenging problem.
\cite{shen2020closed} identifies latent semantics and examine the representation learned by GANs. The researchers derive a closed-form factorization method to discover latent semantic and prove that all semantic directions found are orthogonal to each other in the latent space.
\cite{bahng2019rebias} trains a de-biased
representation by encouraging it to be different from a set of representations that are biased by design.
The method discourages models from taking bias shortcuts, resulting in improved performances on de-biased test data.

Data augmentation 
is a common practice to improve the generalization ability of deep models~\cite{alexnet2012,srivastava2015training,han2017deep}. Augmentation strategies, such as Cutout~\cite{devries2017improved}, Mixup~\cite{zhang2017mixup}, CutMix~\cite{yun2019cutmix} and AugMix~\cite{hendrycks2019augmix}, are able to improve the performance of deep models effectively. One promising data augmentation strategy closely relevant to our work is to interpolate high-level representation learned by the deep model. \cite{upchurch2017deep} showed that simple linear interpolation of features can achieve meaningful semantic transformation on the input image. Inspired by this observation, \cite{wang2019implicit} proposes to use random vectors sampled from a class-specific normal distribution to achieve augmentation of deep features. Instead of augmenting the features explicitly, they minimize an upper bound of the expected loss on augmented data. 
In order to handle few-shot learning problems, \cite{hariharan2017low} proposes to train a feature generator that can transfer modes of variation from categories of a large dataset to novel classes with limited samples. 
To ease the learning from long-tailed data, \cite{liu2020deep} proposed to transfer the intra-class distribution of the head class to the tail class by augmenting the deep features of the instances in the tail class. Different from these methods, our method performs gradient-based augmentation on disentangled context-related features to eliminate distribution shifts for various OoD tasks.

\section{Robustness from Architecture Perspective}

Data distribution mismatches between training and testing set exist in many real-world scenes. Different methods have been developed to tackle OoD shifts. IRM~\cite{arjovsky2019invariant} targets to extract invariant representation from different environments via an invariant risk regularization. IRM-Games~\cite{ahuja2020invariant} aims to achieve the Nash equilibrium among multiple environments to find invariants based on ensemble methods. 
The work of \cite{hendrycks2019using} finds that using pre-training can improve model robustness and uncertainty. However, existing OoD generalization approaches seldom consider the effects of architecture which leads to suboptimal performances. 
EfficientNet~\cite{tan2019efficientnet} proposes a new scaling method that uniformly scales all dimensions of depth, width, and resolution via an effective compound coefficient. 
DARTS~\cite{liu2018darts} presents a differentiable manner to deal with the scalability challenge of architecture search. ISTA-NAS~\cite{yang2020ista} formulates neural architecture search as a sparse coding problem. 
The work \cite{chen2020robustness} uses a robust loss to mitigate the performance degradation under symmetric label noise. However,  NAS overfits easily, the work \cite{yang2019evaluation, guo2020single} points out that NAS evaluation is frustratingly hard. Thus, it is highly non-trivial to extend existing NAS algorithms to the OoD setting.

Recent studies show that different architectures present different generalization abilities.
The work of \cite{zhang2021can} uses a functional modular probing method to analyze deep model structures under the OoD setting.
The work \cite{hendrycks2020pretrained} examines and shows that pre-trained transformers achieve not only high accuracy on in-distribution examples but also improvement of out-of-distribution robustness. The work \cite{dapello2020simulating} presents CNN models with neural hidden layers that better simulate the primary visual cortex improve robustness against image perturbations. The work \cite{dosovitskiy2020image} uses a pure transformer applied directly to sequences of image patches, which performs quite well on image classification tasks compared with relying on CNNs. The work of \cite{dong2020adversarially} targets to improve the adversarial robustness of the network with NAS and achieves superior performance under various attacks. However, they do not consider OoD generalization from the architecture perspective.

\chapter{Out-of-Distribution Generalization via Decomposed Feature Representation and Semantic Augmentation}

\section{Introduction}

Deep learning has demonstrated superior performances on standard benchmark datasets from various fields, such as image classification \cite{alexnet2012}, object detection \cite{redmon2016you}, natural language processing \cite{bert2019}, and recommendation systems \cite{Cheng2016}, assuming that the training and test data are independent and identically distributed (IID).
In practice, however, it is common to observe distribution shifts among training and test data, which is known as out-of-distribution (OoD) generalization.
How to deal with OoD generalization is still an open problem.

\begin{figure}[!t]
    \centering
    \includegraphics[width=0.65\linewidth]{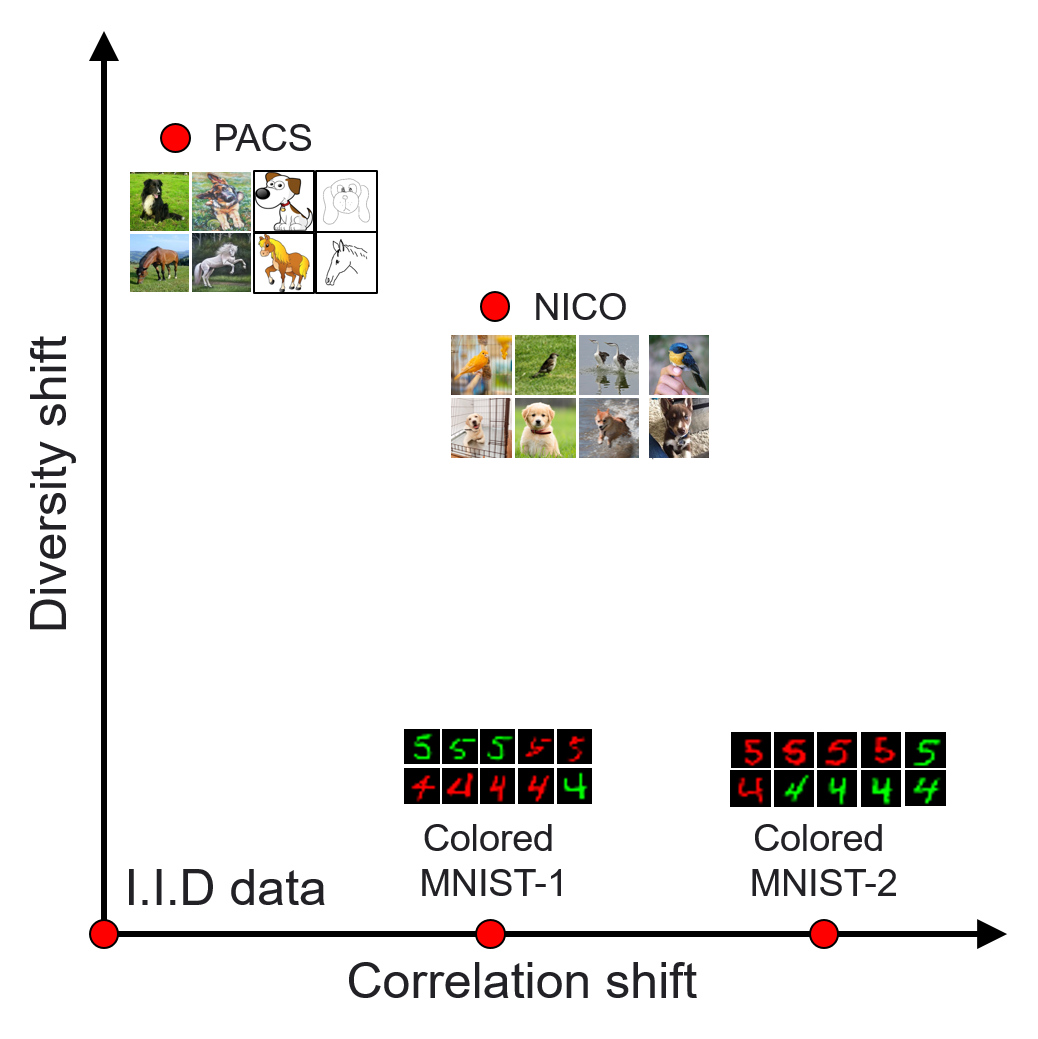}
    \caption{Illustration of the two-dimensional OoD shifts among datasets in different OoD research areas, including Colored MNIST, PACS, and NICO. Extensive experiments showed that many OoD methods can only deal with one dimension of OoD shift.}
    \label{fig:2dimood}
\end{figure}

To improve a DNN's OoD generalization ability, diversified research endeavors are observed recently, which mainly includes domain generalization, invariant risk minimization, and stable learning. Various benchmark datasets are adopted to evaluate the proposed OoD generalization algorithms, such as Colored MNIST~\cite{arjovsky2019invariant}, PACS~\cite{Li2017}, 
and NICO~\cite{he2020towards}.
Among these datasets, PACS are widely used in domain generalization~\cite{carlucci2019domain, mancini2020towards} to validate DNN's ability to generalize across different image styles. 
On the other hand, in recent risk regularization methods, Colored MNIST is often considered \cite{arjovsky2019invariant, ahuja2020invariant, krueger2020outofdistribution, xie2020risk}, where distribution shift is introduced by manipulating the correlation between the colors and the labels.
In stable learning, another OoD dataset called NICO was introduced recently~\cite{he2020towards}, which contains images with various contexts. Along with this dataset, an OoD learning method, named CNBB, is proposed, based on sample re-weighting inspired by causal inference.

In this paper, we observe that methods perform well in one OoD dataset, such as PACS, which may show very poor performance on another dataset, such as Colored MNIST, as shown in our experiments (see experimental results in Sec.~\ref{exp:results}).
That may because of the different types of OoD shifts.
Here, we identify two types of out-of-distribution factors, including the correlation shift and the diversity shift.

One is the correlation shift, which means that labels and environments are correlated and the relations change across different environments. 
We observe the correlation shift between the training set and the test set in Colored MNIST.
Specifically, in training set, the number $5$ is usually in green while the number $4$ is usually in red.
However, in test set, the number $5$ tends to be in red while the number $4$ tends to be in green.
If a model learns color green to predict label $5$ when training, it would suffer from the correlation shift when testing.

\begin{figure}[!t]
\centering
\includegraphics[height = 13.85cm]{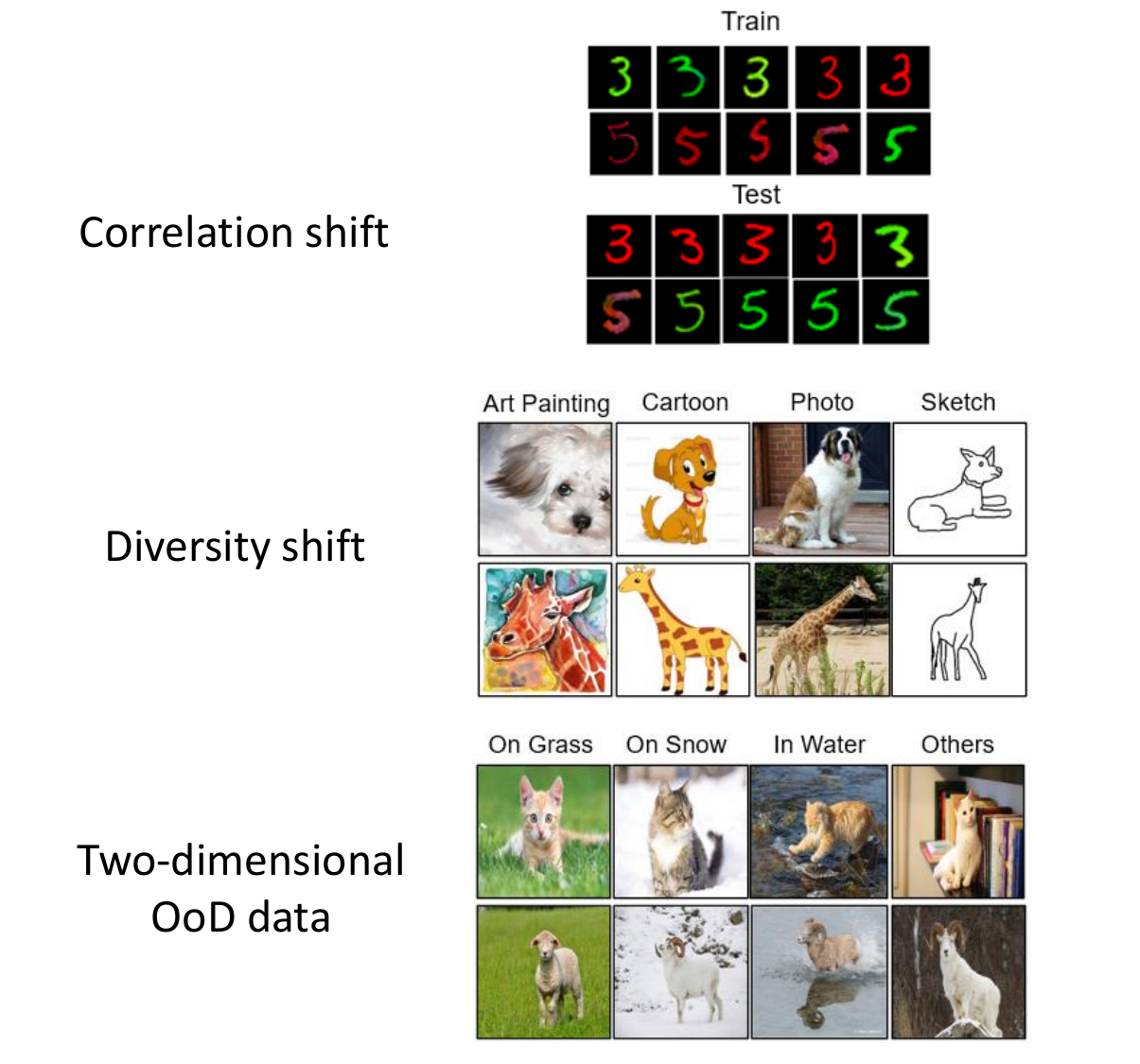}
\caption[Typical examples of the two-dimensional out-of-distribution data from Colored MNIST, PACS, and NICO.]{Typical examples of the two-dimensional out-of-distribution data from Colored MNIST, PACS, and NICO. For the two-dimensional OoD data from the NICO dataset. Contexts such as ``on grass'', ``on snow'' and ``in water'' result in mini-domains in the dataset, suggesting the diversity shift among the data. On the other hand, specific contexts such as ``at home'' are common for cats while is unusual for dogs. The category branch and the context branch are correlated, indicating the correlation shift among data.}
\label{fig:img}
\end{figure}

Another out-of-distribution factor is the diversity shift. For example, in PACS, the data come from four different domains: photo, art painting, cartoon and sketch. Data in different domains have significantly different styles. Usually, we leave one domain out as the test set, and the remaining three domains as the training set. The model trained on the training set would susceptible to the diversity shift on the test set. See Fig.~\ref{fig:img} as an illustration.

Data in actual scenarios usually involve two different OoD factors simultaneously.
For example, in NICO, (Fig.~\ref{fig:img}), 
different contexts such as ``in cage", ``in water", and ``on grass" lead to diversity shift, while some contexts are related to specific categories, such as a bird would be ``in hand'' and a dog may be ``at home''.
We also put datasets from multiple research areas on the same axis, 
the $X$-axis denotes the correlation shift which controls the contribution proportions of correlated features, the $Y$-axis denotes the diversity shift which stands for the change of feature types. Specifically, in the Colored MNIST dataset, the correlation between color and label is high, while in the PACS, the style of images is more diverse. In the NICO, both correlation shift and diversity shift exist.

To handle different OoD factors simultaneously, we propose DecAug, a novel decomposed feature representation and semantic augmentation approach for OoD generalization. 
Specifically, our method first decomposes the high-level representations of input images into category-related and context-related features by orthogonalizing the two gradients of losses for predicting category and context labels respectively. Here, category-related features are essential for recognizing the category labels of the images, while context-related features are not essential for the recognition but correlated with the category labels.
After obtaining the decomposed features, we do gradient-based semantic augmentation on context-related features, representing attributes, styles, backgrounds, or scenes of target objects, to disentangle the spurious correlation between features that are not essential for the recognition and category labels. 

Our contributions are as follows:
\begin{enumerate}
    \item We test OoD methods from diversified research areas and show that very often, they only deal with one special type of OoD generalization challenge. 
    \item We propose \name{} to learn disentangled features that capture the information of category and context respectively and perform gradient-based semantic augmentation to enhance the generalization ability of the model. 
    \item Extensive experiments show that our method consistently outperforms previous OoD methods on various types of OoD tasks. For instance, we achieve an average accuracy of 82.39\% with ResNet-18~\cite{he2016deep} on PACS~\cite{Li2017}, which is the state-of-the-art performance.
\end{enumerate}

\section{Related Work}
\label{sec:decaug-rela}
In this section, we review literature related to risk regularization methods, domain generalization, stable learning, data augmentation and disentangled representation.

\subsection{Risk Regularization Methods for OoD Generalization}
The invariant risk minimization (IRM, \citet{arjovsky2019invariant}) is motivated by the theory of causality and causal Bayesian networks (CBNs), aiming to find an invariant representation of data from different training environments. 
To make the model robust to unseen interventions, the invariant risk minimization added invariant risk regularization to monitor the optimality of a dummy classifier on different environments. 
IRM-Games~\cite{ahuja2020invariant} further improves the stability of IRM.
Risk extrapolation (Rex, \citet{krueger2020outofdistribution}) adopts a min-max framework to derive a model that can perform well on the worst linear combination of risks from different environments. 
These methods typically perform well on synthetic datasets, such as Colored MNIST. However, it is unknown how they can generalize on more complex practical datasets beyond MNIST classification tasks.

\subsection{Domain Generalization}
\citet{carlucci2019domain} proposed a self-supervised learning method for typical domain generalization datasets, such as PACS, by solving Jigsaw puzzles. \citet{dou2019domain} adopted meta-learning to learn invariant feature representations across domains. Recently, \citet{mancini2020towards} proposed the curriculum mixup method for domain generalization, in which data from multiple domains in the training dataset mix together by a curriculum schedule of mixup method. Domain generalization methods have achieved performance gain in generalizing models to unseen domains. However, recent OoD research finds that domain adaptation methods with similar design principles can have problems when training distribution is largely different from test distribution \cite{arjovsky2019invariant}.

Stable learning is a recently proposed new concept \cite{Kuang2018}, which focuses on learning a model that can achieve stable performances across different environments. The methodology of stable learning largely inherited from sampler reweighting in causal inference \cite{Kuang2018, shen2019stable, he2020towards}. 
While these methods can have theoretical guarantees on simplified models, when confounder results in strong spurious correlations, this method may not be able to work well especially in the deep learning paradigm.

\subsection{Data Augmentation Approaches}

Data augmentation has been widely used in deep learning to improve the generalization ability of deep models~\cite{alexnet2012,srivastava2015training,han2017deep}. Elaborately designed augmentation strategies, such as Cutout~\cite{devries2017improved}, Mixup~\cite{zhang2017mixup}, CutMix~\cite{yun2019cutmix}, and AugMix~\cite{hendrycks2019augmix}, have effectively improved the performance of deep models. A more related augmentation method is to interpolate high-level
representations. \citet{upchurch2017deep} shows that simple linear interpolation can achieve meaningful semantic transformations. Motivated by this observation, \citet{wang2019implicit} proposes to augment deep features with random vectors sampled from class-specific normal distributions. Instead of augmenting the features explicitly, they minimize an upper bound of the expected loss on augmented data. To tackle the few-shot learning problem, \citet{hariharan2017low} suggest training a feature generator that can transfer modes of variation from categories of a large dataset to novel classes with limited samples. To ease the learning from long-tailed data, \citet{liu2020deep} proposes to transfer the intra-class distribution of head classes to tail classes by augmenting deep features of instances in tail classes. Different from these approaches, our method performs gradient-based augmentation on disentangled context-related features to eliminate distribution shifts for various OoD tasks.

Disentangling the latent factors from the image variants is a promising way to provide an understanding of the observed data \cite{chen2016infogan,higgins2017beta,ma2019disentangled}.
It aims to learn representations that separate the explanatory factors of variations behind the data. Such representations are more resilient to the complex variants and able to bring enhanced generalization ability \cite{liu2018detach,peng2019domain}.
Disentangled representations are inherently more interpretable.
How to obtain disentanglement is still a challenging problem.
\citet{shen2020closed} identifies latent semantics and examines the representation learned by GANs.
\citet{bahng2019rebias} trains a de-biased
representation by encouraging it to be different from a set of representations that are biased by design.
In this paper, semantic vectors found by DecAug with orthogonal constraints are disentangled from each other in the feature space.

\section{Methodology}

We argue that it is essential to obtain the disentangled features to address the aforementioned two-dimension OoD generalization simultaneously, one is the target for recognition and the other is not critical but correlated for recognition.
In this section, we introduce \name{}, a novel decomposed feature representation and semantic augmentation approach to learn the disentangled high-level representation in the feature space. The decomposition is achieved based on the orthogonalization constrain of the two gradients of the losses for predicting category and context labels, as shown in Figure~\ref{fig:framework}. We do gradient-based semantic augmentation on context-related features to enhance the generalization of the model. The augmented features and category-related features are concatenated to make the final prediction.

\begin{figure}[!t]
    \centering
    \includegraphics[width=1.0\linewidth]{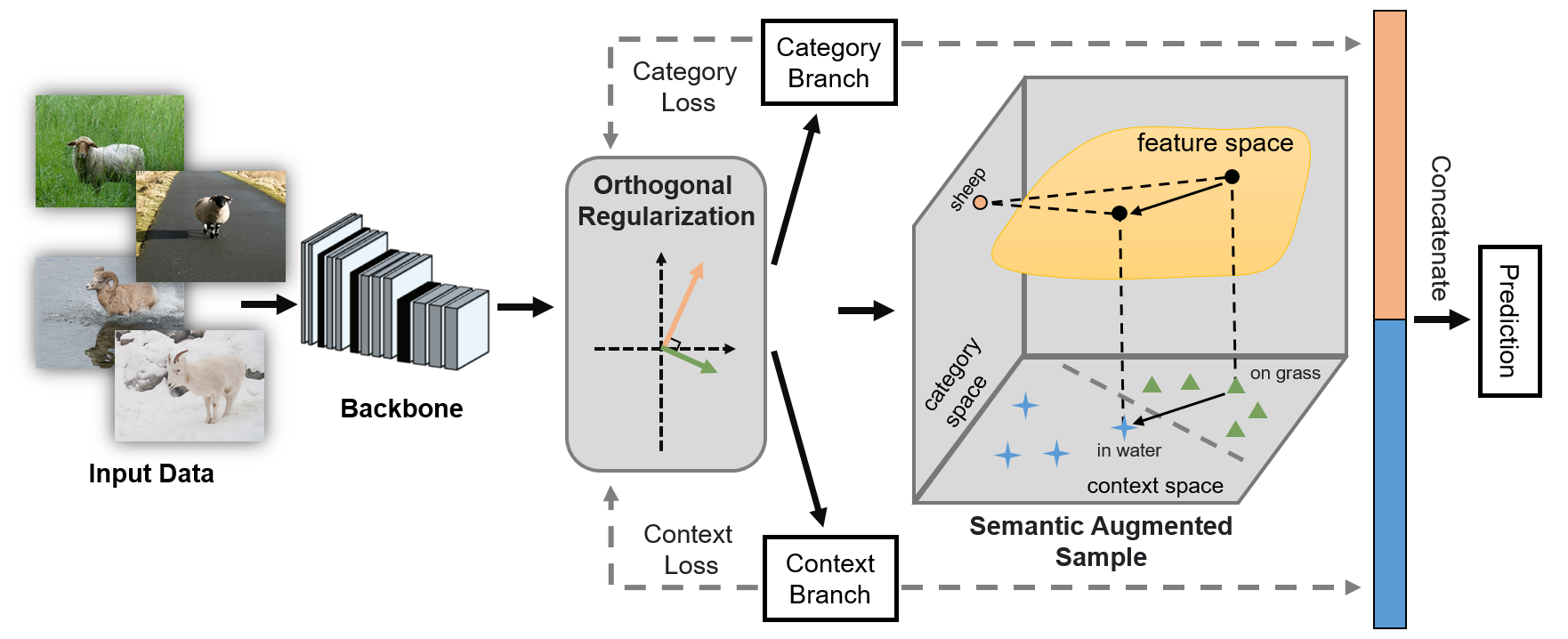}
    \caption{An overview of the proposed \name{}. The input features $z$ extracted by the backbone are decomposed into category-related and context-related features with orthogonal regularization. Gradient-based augmentation is then preformed in the feature space to get semantic augmented samples.}
    \label{fig:framework}
\end{figure}

\subsection{Feature Decomposition}

As shown in Figure~\ref{fig:framework}, we consider an image recognition task with the training set $\mathcal{D} = \{ (x_i, y_i, c_i) \}_{i=1}^{N}$, where $x_i$ is the input image, $y_i$ is the corresponding category label, $c_i$ is the corresponding context label, and $N$ is the number of training data. The input data are feed to the network and mapped to the feature space. The high-level representation are disentangled into two parts: category-related features and context-related features. Given an input image $x_i$ with category label $y_i$ and context label $c_i$, let $z_i=g_{\theta}(x_i)$ denotes the features extracted by a backbone $g_{\theta}$. For the category branch, $z_i$ is decomposed into $z^1_i = f_{\theta^1}(z_i)$ by a category feature extractor $f_{\theta^1}$, followed by a classifier $h_{\phi^1}(z^1_i)$ to predict the category label. For the context branch, $z_i$ is decomposed into $z^2_i = f_{\theta^2}(z_i)$ by a context feature extractor $f_{\theta^2}$, followed by a classifier $h_{\phi^2}(z^2_i)$ to predict the context label. We leverage the standard cross-entropy losses $\mathcal{L}^1_i(\theta, \theta^1, \phi^1) = \ell(h_{\phi^1}\circ f_{\theta^1}(z_i), y_i)$ and $\mathcal{L}^2_i(\theta, \theta^2, \phi^2) = \ell(h_{\phi^2}\circ f_{\theta^2}(z_i), c_i)$ to optimize these two branches, together with the backbone, respectively.

\begin{algorithm}[!ht]\small
\caption{\name{}: Decomposed Feature Representation and Semantic Augmentation for OoD generalization}
\label{alg:method}
\begin{algorithmic}[1]
\REQUIRE Training set $\mathcal{D}$, batch size $n$, learning rate $\beta$, hyper-parameters $\epsilon$, $\lambda^1$, $\lambda^2$, $\lambda^{\text{orth}}$.
\ENSURE $\theta$, $\theta^1$, $\phi^1$, $\theta^2$, $\phi^2$, $\phi$.
\STATE Initialize $\theta$, $\theta^1$, $\phi^1$, $\theta^2$, $\phi^2$, $\phi$;
\REPEAT
\STATE Sample a mini-batch of training images $\{(x_i, y_i, c_i)\}_{i=1}^{n}$ with batch size $n$;
\FOR{each $(x_i, y_i, c_i)$}
\STATE $z_i \leftarrow g_{\theta}(x_i)$;
\STATE $\mathcal{L}^1_i(\theta, \theta^1, \phi^1) \leftarrow \ell(h_{\phi^1}\circ f_{\theta^1}(z_i), y_i)$;
\STATE $\mathcal{L}^2_i(\theta, \theta^2, \phi^2) \leftarrow \ell(h_{\phi^2}\circ f_{\theta^2}(z_i), c_i)$;
\STATE Compute $\mathcal{L}^{\text{orth}}_i(\theta^1, \phi^1, \theta^2, \phi^2)$ according to Eq.~\eqref{equ:orthloss};
\STATE Randomly sample $\alpha_i$ from $[0,1]$;
\STATE Generate $\tilde{z}^2_i$ according to Eq.~\eqref{equ:featureaug};
\STATE $\mathcal{L}^{\text{concat}}_i(\theta, \theta^1, \theta^2, \phi) \leftarrow \ell(h_{\phi}([f_{\theta^1}(z_i), \tilde{z}^2_i]),y_i)$;
\STATE Compute $\mathcal{L}_i(\theta, \theta^1, \phi^1, \theta^2, \phi^2, \phi)$ according to Eq.~\eqref{equ:finalloss};
\ENDFOR
\STATE $(\theta, \theta^1, \phi^1, \theta^2, \phi^2, \phi) \leftarrow (\theta, \theta^1, \phi^1, \theta^2, \phi^2, \phi)$ \\ \rightline{$- \beta \cdot \nabla \frac{1}{n} \sum\limits_{i=1}^n \mathcal{L}_i(\theta, \theta^1, \phi^1, \theta^2, \phi^2, \phi)$;}
\UNTIL convergence;
\end{algorithmic}
\end{algorithm}

For the orthogonality in the feature space between category-related feature $z^1_i$ and context-related feature $z^2_i$.
It is obvious that the direction of the non-zero gradient of a function is the direction in which the function increases most quickly, while the direction that is orthogonal to the gradient direction is the direction in which the function does not change.

We enforce the gradient of the category loss $\ell(h_{\phi^1}\circ f_{\theta^1}(z_i), y_i)$ to be orthogonal to the gradient of the context loss $\ell(h_{\phi^2}\circ f_{\theta^2}(z_i), c_i)$ with respect to $z_i$ for better disentangling the category-related and context-related features.
In this way, the direction that changes the category loss most quickly will not change the context loss from $z_i$ and vice versa. Specifically, let $\mathcal{G}^1_i(\theta^1, \phi^1) = \nabla_{z_i} \ell(h_{\phi^1}\circ f_{\theta^1}(z_i), y_i)$ and $\mathcal{G}^2_i(\theta^2, \phi^2) = \nabla_{z_i} \ell(h_{\phi^2}\circ f_{\theta^2}(z_i), c_i)$ be the gradients of the category and context loss with respect to $z_i$ respectively. To ensure the orthogonality, we minimize the following loss:
\small
\begin{center}
\begin{equation}\label{equ:orthloss}
\mathcal{L}^{\text{orth}}_i(\theta^1, \phi^1, \theta^2, \phi^2) = (\frac{\mathcal{G}^1_i(\theta^1, \phi^1)}{\left\lVert \mathcal{G}^1_i(\theta^1, \phi^1) \right\rVert} \cdot \frac{\mathcal{G}^2_i(\theta^2, \phi^2)}{\left\lVert \mathcal{G}^2_i(\theta^2, \phi^2) \right\rVert})^2.
\end{equation}
\end{center}
\normalsize

\subsection{Semantic Augmentation}

After obtaining the decomposed category-related and context-related features, we conduct gradient-based semantic augmentation on the context-related features to eliminate the spurious correlation between features. We may have multiple alternative directions for OoD in the semantic feature space, we assume a worse case for the model to learn for OoD generalization by obtaining the adversarially perturbed examples in the feature space to ensure good performance across different environments.
To be specific, let $\mathcal{G}^{\text{aug}}_i = \nabla_{z^2_i} \ell(h_{\phi^2}(z^2_i), c_i)$ be the gradient of the context loss with respect to $z^2_i$. We augment the context-related features $z^2_i$ as follows:
\small
\begin{center}
\begin{equation}\label{equ:featureaug}
\tilde{z}^2_i = z^2_i + \alpha_i \cdot \epsilon \cdot \frac{\mathcal{G}^{\text{aug}}_i}{\left\lVert \mathcal{G}^{\text{aug}}_i \right\rVert},
\end{equation}
\end{center}
\normalsize
where $\alpha_i$ is randomly sampled from $[0,1]$ and $\epsilon$ is a hyper-parameter that determines the maximum length of the augmentation vectors

The augmented $\tilde{z}^2_i$ and category-related features $z^1_i$ are concatenated to make the final prediction $h_{\phi}([z^1_i, \tilde{z}^2_i])$, where $[z^1_i, \tilde{z}^2_i]$ is the concatenation of two features and $h_{\phi}$ is a classifier. The corresponding parameters are optimized by the cross-entropy loss $\mathcal{L}^{\text{concat}}_i(\theta, \theta^1, \theta^2, \phi) = \ell(h_{\phi}([z^1_i, \tilde{z}^2_i]),y_i)$. Together with the aforementioned losses, the final loss is formulated as follows:

\small
\begin{equation}\label{equ:finalloss}
\begin{aligned}
\mathcal{L}_i(\theta, \theta^1, \phi^1, \theta^2, \phi^2, \phi) &= \mathcal{L}^{\text{concat}}_i(\theta, \theta^1, \theta^2, \phi) \\
&+ \lambda^1 \cdot \mathcal{L}^1_i(\theta, \theta^1, \phi^1) + \lambda^2 \cdot \mathcal{L}^2_i(\theta, \theta^2, \phi^2)\\
&+ \lambda^{\text{orth}} \cdot \mathcal{L}^{\text{orth}}_i(\theta^1, \phi^1, \theta^2, \phi^2),
\end{aligned}
\end{equation}
\normalsize
where $\lambda^1$, $\lambda^2$ and $\lambda^{\text{orth}}$ are hyper-parameters that balance different losses. We define the learning of DecAug as the following optimization problem:
\small
\begin{equation}\label{equ:objective}
\min_{\theta, \theta^1, \phi^1, \theta^2, \phi^2, \phi}\, \frac{1}{N}\sum_{i=1}^{N} \mathcal{L}_i(\theta, \theta^1, \phi^1, \theta^2, \phi^2, \phi).
\end{equation}
\normalsize

We summarize the detailed procedures in Algorithm~\ref{alg:method}. And the above objective function is optimized by the stochastic gradient descent (SGD) algorithm.

\subsection{Theoretical Justification of DecAug}

OoD generalization problem is quite complex for theoretical analysis in general case. However, we provide a simplified case to analytically demonstrate why gradient orthogonalization can help. Let us consider the structural equation model (SEM):
\small
\begin{align}
    X_1  &\leftarrow N_1(0, \sigma^2),\nonumber\\
    Y    &\leftarrow X_1 + N_2(0, \sigma^2),\nonumber\\
    X_2  &\leftarrow Y + N_3(0, 1),\nonumber\\
    C    &\leftarrow X_2 + N_4(0, \sigma^2),
\end{align}
\normalsize
where $N_1, N_2, N_3, N_4$ are independent Gaussian distributions. $Y$ and $C$ represent labels for category and context respectively, whereas $X_1$ and $X_2$ represent features for predicting $Y$ and $C$. More specifically, $X_1$ is the causal feature for predicting $Y$, and $X_2$ is the causal feature for predicting $C$, however $X_2$ is also spuriously correlated with $Y$.

Suppose we have a dataset $D=\{(x_1, x_2, y, c)_i\}_{i=1}^{n}$, which contains examples identically and independently distributed according to $P(X_1, X_2, Y, C)$ that is consistent with the above SEM. Now we predict $y$ and $c$ from $(x_1, x_2)$ using two linear predictors:
\small
\begin{align}
    \Hat{y} &= x_1 \alpha_{1,1} + x_2 \alpha_{2,1},\\
    \Hat{c} &= x_1 \alpha_{1,2} + x_2 \alpha_{2,2},
\end{align}
\normalsize
which are learned by minimizing the squared errors. For every training example, $\mathcal{L}^\text{orth}$ is then defined as
\small
\begin{equation}
    \left(\frac{\nabla_{x_1, x_2}(\Hat{y}-y)^2}{\lVert \nabla_{x_1, x_2}(\Hat{y}-y)^2 \rVert} \cdot \frac{\nabla_{x_1, x_2}(\Hat{c}-c)^2}{\lVert \nabla_{x_1, x_2}(\Hat{c}-c)^2 \rVert}\right)^2,
\end{equation}
\normalsize
following the idea of gradient orthogonalization that we propose in DecAug. We assume in the following theorem that $(\Hat{y}-y)(\Hat{c}-c) \neq 0$ for every $(x_1, x_2, y, c) \in D$, otherwise the computation of $\mathcal{L}^\text{orth}$ would involve division by zero. In practice, we can simply ignore these examples while computing the loss over a batch.

\begin{theorem}\label{theorem:alpha21}
    For arbitrary $\sigma \geq 0$, $\mathbb{E}[\mathcal{L}^2]$ and $\mathbb{E}[\mathcal{L}^\textnormal{orth}]$ are both minimized only if $\Hat{y}$ does not predict $y$ from $x_2$, i.e., $\alpha_{2,1} = 0$, even if there is spurious correlation between $X_2$ and $Y$.
\end{theorem}
Here we briefly sketch the proof of Theorem \ref{theorem:alpha21}. For the full proof, please refer to Appendix. First, by the definition of $\mathcal{L}^2$ and the SEM given above, we have
\small
\begin{equation}\label{eq:ELc}
    \mathbb{E}[\mathcal{L}^2] = (\alpha_{1,2}+\alpha_{2,2}-1)^2\sigma^2 + (\alpha_{2,2}-1)^2(\sigma^2 + 1) + \sigma^2
\end{equation}
\normalsize
which implies that the minimum value of $\mathbb{E}_D[\mathcal{L}^2]$ is $\sigma^2$, obtained only when $\alpha_{1,2} = 0$ and $\alpha_{2,2} = 1$. Combining this with the definition of $\mathbb{E}[\mathcal{L}^\text{orth}]$, we have
\small
\begin{equation}\label{eq:ELorth}
    \mathbb{E}[\mathcal{L}^{\text{orth}}] = \frac{(\alpha_{1,1}\alpha_{1,2} + \alpha_{2,1}\alpha_{2,2})^2}{(\alpha_{1,1}^2 + \alpha_{2,1}^2)(\alpha_{1,2}^2 + \alpha_{2,2}^2)} = \frac{\alpha_{2,1}^2}{\alpha_{1,1}^2+\alpha_{2,1}^2}
\end{equation}
\normalsize
where the second equality follows from the condition implied by $\mathbb{E}[\mathcal{L}^2] = \sigma^2$. Finally, \eqref{eq:ELorth} is minimized only when $\alpha_{2,1} = 0$.

We provide the proof of theorem for disentangled representations as follows:
\begin{lemma}\label{lemma:ELc-bound}
    $\mathbb{E}[\mathcal{L}^2] \geq \sigma^2$ and the equality holds if and only if $\alpha_{1,2} = 0$ and $\alpha_{2,2} = 1$.
\end{lemma}
\begin{proof}
    First, we expand $\mathbb{E}[\mathcal{L}^2]$ to find out the expectation:
    \small
    \begin{equation}\label{eq:ELc_2}
        \begin{split}
               \mathbb{E}[\mathcal{L}^2] =& \mathbb{E}[( \Hat{c}-c )^2]\\
            =& \mathbb{E}[( x_1 \alpha_{1,2} + x_2 \alpha_{2,2} - c )^2]\\
            =& \mathbb{E}[( N_1\alpha_{1,2} + (N_1 + N_2 + N_3)\alpha_{2,2}\\
            & - (N_1 + N_2 + N_3 + N_4) )^2]\\
            =& \mathbb{E}[( (\alpha_{1,2}+\alpha_{2,2}-1)N_1 + (\alpha_{2,2}-1)N_2\\
            & + (\alpha_{2,2}-1)N_3 - N_4 )^2]\\
            =& (\alpha_{1,2}+\alpha_{2,2}-1)^2\sigma^2 + (\alpha_{2,2}-1)^2(\sigma^2 + 1) + \sigma^2
        \end{split}
    \end{equation}
    \normalsize
    where the last equality follows from the independence of the Gaussian distributions $N_1$, $N_2$, $N_3$, $N_4$. Since all three terms in the result of \eqref{eq:ELc} are non-negative, it is easy to see that $\mathbb{E}[\mathcal{L}^2] \geq \sigma^2$ and the equality holds if and only if 
    \small
    \begin{align}\label{eq:ELc-eq-set}
        (\alpha_{1,2} + \alpha_{2,2} - 1)^2 &= 0,\nonumber\\
        (\alpha_{2,2} - 1)^2 &= 0.
    \end{align}
    \normalsize
    Solving \eqref{eq:ELc-eq-set} finishes the proof.
\end{proof}

\noindent\textbf{Theorem 1} \textit{
    For arbitrary $\sigma \geq 0$, $\mathbb{E}[\mathcal{L}^2]$ and $\mathbb{E}[\mathcal{L}^\textnormal{orth}]$ are both minimized only if $\Hat{y}$ does not predict $y$ from $x_2$, i.e., $\alpha_{2,1} = 0$, even if there is spurious correlation between $X_2$ and $Y$.
}
\begin{proof}
    By the definition of $\mathcal{L}^\textnormal{orth}$, we have
    \small
    \begin{equation}\label{eq:ELorth_2}
        \mathbb{E}[\mathcal{L}^\textnormal{orth}] = \mathbb{E}[(\frac{\nabla_{x_1, x_2}(\Hat{y}-y)^2}{\lVert \nabla_{x_1, x_2}(\Hat{y}-y)^2 \rVert} \cdot \frac{\nabla_{x_1, x_2}(\Hat{c}-c)^2}{\lVert \nabla_{x_1, x_2}(\Hat{c}-c)^2 \rVert})^2].
    \end{equation}
    \normalsize
    Since $\Hat{y} = x_1 \alpha_{1,1} + x_2 \alpha_{2,1}$ and $\Hat{c} = x_1 \alpha_{1,2} + x_2 \alpha_{2,2}$, the gradients in \eqref{eq:ELorth} can be written as
    \small
    \begin{alignat}{5}
        &\nabla_{x_1, x_2}(\Hat{y}-y)^2 &&= (2(\Hat{y}-y) &&\alpha_{1,1},\  2(\Hat{y}-y) &&\alpha_{2,1}),\nonumber\\
        &\nabla_{x_1, x_2}(\Hat{c}-c)^2 &&= (2(\Hat{c}-c) &&\alpha_{1,2},\  2(\Hat{c}-c) &&\alpha_{2,2}).
    \end{alignat}
    \normalsize
    Therefore,
    \small
    \begin{equation}\label{eq:ELorth-alpha}
        \begin{split}
            &\mathbb{E}[\mathcal{L}^\textnormal{orth}] \\
            &= \mathbb{E}[(\frac{4(\Hat{y}-y)(\Hat{c}-c)\alpha_{1,1}\alpha_{2,1} + 4(\Hat{y}-y)(\Hat{c}-c)\alpha_{1,2}\alpha_{2,2}}{\sqrt{4(\Hat{y}-y)^2(\alpha_{1,1}^2+\alpha_{2,1}^2)}\cdot\sqrt{4(\Hat{c}-c)^2(\alpha_{1,2}^2+\alpha_{2,2}^2)}})^2]\\
            &= \mathbb{E}[\frac{(\alpha_{1,1}\alpha_{2,1}+\alpha_{1,2}\alpha_{2,2})^2}{(\alpha_{1,1}^2+\alpha_{2,1}^2)(\alpha_{1,2}^2+\alpha_{2,2}^2)}]\\
            &= \frac{(\alpha_{1,1}\alpha_{1,2} + \alpha_{2,1}\alpha_{2,2})^2}{(\alpha_{1,1}^2 + \alpha_{2,1}^2)(\alpha_{1,2}^2 + \alpha_{2,2}^2)}.
        \end{split}
    \end{equation}
    \normalsize
    By Lemma \ref{lemma:ELc-bound}, $\mathbb{E}[\mathcal{L}^2]$ is minimized if and only if $\alpha_{1,2} = 0$ and $\alpha_{2,2} = 1$, so
    \small
    \begin{equation}\label{eq:alpha-reduce}
        \frac{(\alpha_{1,1}\alpha_{1,2} + \alpha_{2,1}\alpha_{2,2})^2}{(\alpha_{1,1}^2 + \alpha_{2,1}^2)(\alpha_{1,2}^2 + \alpha_{2,2}^2)} = \frac{\alpha_{2,1}^2}{\alpha_{1,2}^2 + \alpha_{2,2}^2}.
    \end{equation}
    \normalsize
    Combining \eqref{eq:ELorth-alpha} and \eqref{eq:alpha-reduce}, we conclude that if $\mathbb{E}[\mathcal{L}^\textnormal{orth}]$ is also minimized at the same time then $\alpha_{2,1} = 0$.
\end{proof}

\section{Illustrative Results}
In this section, we compare DecAug with various methods for image recognition on different OoD datasets: Colored MNIST~\cite{arjovsky2019invariant}, Rotated MNIST~\cite{ghifary2015domain}, PACS~\cite{Li2017}, VLCS~\cite{fang2013unbiased}, and NICO~\cite{he2020towards}. We show that DecAug can deal with different types of OoD generalization challenges and achieve superior performance.

\subsection{Datasets and Implementation Details}

The challenging Colored MNIST dataset was recently proposed by IRM~\cite{arjovsky2019invariant} via modifying the original MNIST dataset with three steps: 1) The original digits ranging from 0 to 4 were relabelled as 0 and the digits ranging from 5 to 9 were tagged as 1; 2) The labels of 0 have a  probability of $25\%$ to flip to 1, and vice versa; 3) The digits were colored either red or green based on different correlation with the labels to construct different environments (e.g., $80\%$ and $90\%$ for the training environments and $10\%$ for the test environment).
In this way, the classifiers will easily over-fit to the spurious feature (e.g., color) in the training environments and ignore the shape feature of the digits.

For a fair comparison, we followed the same experimental protocol as in IRM~\cite{arjovsky2019invariant} on the Colored MNIST dataset.
We equipped the IRMv1 scheme with our \name{} approach
using the same settings.
The backbone network was a three-layer MLP. 
The total training epoch was 500 and the batch size was the whole training data. We used the SGD optimizer with an initial learning rate of 0.1.
The trained model was tested at the final epoch.

This dataset is a variant of the MNIST dataset, where we split the dataset into 3 environments and rotate different angles in each environment. For the training environment, it contains 2 environments, we rotate digits for $60^{\circ}$ in the first environment and $90^{\circ}$ in the second environment. For the testing environment, we rotate digits for $180^{\circ}$. Same as Colored MNIST, The original digits ranging from 0 to 4 were relabelled as 0 and the digits ranging from 5 to 9 were tagged as 1. The labels of 0 have a probability of 25\% to flip to 1, and vice versa. 
In this dataset, we also followed the same experimental protocol as in Colored MNIST dataset.

This dataset contains 4 domains (Photo, Art Painting, Cartoon, Sketch) with 7 common categories (dog,
elephant, giraffe, guitar, horse, house, person). We followed the same leave-one-domain-out validation experimental protocol as in~\cite{Li2017}. For each time, we select three environments for training and the remaining environment for testing.

The backbone network we used on the PACS dataset was the pre-trained ImageNet model ResNet-18. We followed the same training, validation and test split as in JiGen~\cite{carlucci2019domain}.

VLCS contains images of 5 object categories from 4 separated domains.
The backbone network was pre-trained ImageNet model ResNet-18 . We followed the same training, validation and test split as in JiGen~\cite{carlucci2019domain}. The number of training epochs was 100 and the optimizer is SGD. For each algorithm, we use hyperparameter optimization for tuning the model. We conducted a random search of 10 trials over the hyperparameter distribution in DomainBed~\cite{gulrajani2020search},  selected hyperparameters with best test accuracy, and report the mean test accuracy after 3 runs using chosen hyperparameters.

\begin{table}[!t]
    \centering
    \caption[Results of DecAug compared with different methods on the Colored MNIST dataset.]{Results of our approach compared with different methods on the Colored MNIST dataset (mean $\pm$ std deviation).} 
    \label{table:cmnist}
	\begin{adjustbox}{max width=0.55\textwidth}
	\begin{threeparttable}
        \begin{tabular}{ll}
            \toprule
            \toprule
            Model & Acc test env\\
            \midrule
            ERM~\cite{arjovsky2019invariant}  &  17.10 $\pm$ 0.6 \\
            IRM~\cite{arjovsky2019invariant} &  66.90 $\pm$ 2.5 \\
            REx~\cite{krueger2020outofdistribution} &  68.70 $\pm$ 0.9 \\
            F-IRMGames~\cite{ahuja2020invariant}  & 59.91 $\pm$ 2.7 \\
            V-IRMGames~\cite{ahuja2020invariant} & 49.06 $\pm$ 3.4 \\
            ReBias~\cite{bahng2019rebias}\tnote{*} & 29.40 $\pm$ 0.3\\
            JiGen~\cite{carlucci2019domain}\tnote{*}  &11.91 $\pm$ 0.4\\
            Mixup~\cite{zhang2017mixup}\tnote{*} & 28.57 $\pm$ 5.3 \\
            DANN~\cite{ganin2016domain}\tnote{*} &  25.20 $\pm$ 1.2 \\
            \midrule
            \emph{DecAug} & \emph{\textbf{69.60 $\pm$ 2.0}}\\
            ERM, grayscale model (oracle) &  73.00 $\pm$ 0.4  \\
            Optimal invariant model (hypothetical) &75.00 \\
            \bottomrule
            \bottomrule
        \end{tabular}
        \begin{tablenotes}
		    \item[*] Implemented by ourselves.
		\end{tablenotes}
    \end{threeparttable}
    \end{adjustbox}
    \vspace{-0.3cm}
\end{table}

This dataset contains 19 classes with 9 or 10 different contexts, 
i.e., different object poses, positions, backgrounds, and movement patterns, etc. The NICO dataset is one of the newly proposed OoD generalization benchmarks in the real scenarios~\cite{he2020towards}.
The contexts in the validation and test set will not appear in the training set.

The backbone network was ResNet-18 without pretraining on the NICO dataset. The number of training epochs was 500 and the batch size was 128.
We used the SGD optimizer with a learning rate of 0.05. Similar to the VLCS dataset, we also used hyperparameter optimization for tuning the model.

We compare our proposed \name{} with the state-of-the-arts, including empirical risk minimization  (ERM), risk regularization (IRM~\cite{arjovsky2019invariant}, REx~\cite{krueger2020outofdistribution}, IRM-Games~\cite{ahuja2020invariant}), meta learning 
(MASF~\cite{dou2019domain}, MLDG~\cite{li2018learning}), self-supervised learning  (JiGen~\cite{carlucci2019domain}), adversarial feature learning (CORAL~\cite{sun2016deep},
DANN~\cite{ganin2016domain},
MMD~\cite{li2018domain}
), robustness method (DRO~\cite{sagawa2019distributionally}), data augmentation  (Mixup~\cite{zhang2017mixup}, CuMix~\cite{mancini2020towards}), transfer learning (MTL~\cite{blanchard2017domain}),
debiased training  (ReBias~\cite{bahng2019rebias}), and convnets with batch balancing  (CNBB~\cite{he2020towards}) across multiple datasets.

Our framework was implemented with PyTorch 1.1.0, CUDA v9.0.
For the baseline methods, we implement either with Pytorch 1.1.0 or with Tensorflow 1.8 to keep the same setting as their original source code. IRM-Games and MASF were implemented with Tensorflow.The remaining algorithms were implemented with Pytorch.
We conducted experiments on NVIDIA Tesla V100. 

\begin{table}[!t]
    \centering
    \caption[Results of DecAug compared with existing approaches on the Rotated MNIST dataset.]{Results of our approach compared with different methods on the Rotated MNIST dataset (mean $\pm$ std deviation). All methods are implemented by ourselves.} 
    \label{table:rmnist}
	\begin{adjustbox}{max width=0.75\textwidth}
	\begin{threeparttable}
        \begin{tabular}{ll}
            \toprule
            \toprule
            Model & Acc test env\\
            \midrule
            ERM~\cite{arjovsky2019invariant}  &  52.84 $\pm$ 1.2 \\
            IRM~\cite{arjovsky2019invariant} &  53.23 $\pm$ 1.1 \\
            REx~\cite{krueger2020outofdistribution} &52.89 $\pm$ 1.6 \\
            F-IRMGames~\cite{ahuja2020invariant}  & 50.32 $\pm$ 2.7 \\
            V-IRMGames~\cite{ahuja2020invariant} & 50.49 $\pm$ 0.5 \\
            JiGen~\cite{carlucci2019domain}  &50.32 $\pm$ 2.8\\
            Mixup~\cite{zhang2017mixup} &  51.88 $\pm$ 1.2	 \\
            MLDG~\cite{li2018learning}   &  50.81 $\pm$ 1.4	 \\
            DRO~\cite{sagawa2019distributionally}    &  51.94 $\pm$ 1.1	 \\
            \midrule
            \emph{DecAug} & \emph{\textbf{53.87 $\pm$ 0.9}}\\
            \bottomrule
            \bottomrule
        \end{tabular}
        
    \end{threeparttable}
    \end{adjustbox}
    \vspace{-0.3cm}
\end{table}

\subsection{Results and Discussion}
\label{exp:results}
In this section, the results of our approach on five datasets: Colored MNIST, Rotated MNIST, PACS, VLCS and NICO will be evaluated and analyzed.
To provide more thorough studies on OoD generalization compared with previous work, these five datasets represent different aspects of covariant shifts including diversity shift and correlation shift in OoD problems.

DecAug achieves the best generalization performance on Colored MNIST as shown in Table~\ref{table:cmnist}, followed by REx and IRM which are risk regularization methods. As for representative domain generalization methods, such as JiGen and MLDG, they fail to generalize in the test environment for they are misled by the spurious correlation existing in the training datasets. In the grayscale MNIST dataset, DecAug's performance is very close to ERM, which provides an upper bound for the MLP network to generalize on this task. Typical generalization methods only consider dealing with diversity shift, which is one dimension in OoD problem where image style differs, as mentioned in \cite{arjovsky2019invariant}. However, our proposed method, by using decomposition and semantic augmentation in the feature space, greatly improves the performance in Colored MNIST. This is because decomposition and semantic augmentation will disregard spurious features that are correlated but not causal for predicting category.

In the Rotated MNIST dataset, DecAug achieves the best performance followed by IRM and REx. The result is shown in Table \ref{table:rmnist}. Different from results on the  Colored MNIST dataset, ERM can generalize well to the third environment with an accuracy of 52\% in the test environment. The typical risk regularization method obtains slightly better performance due to the correlation between rotation and labels are detected by these algorithms. As for JiGen, it shuffles parts of the original figure and thus is more sensitive to changes of rotation.
For our proposed DecAug, it decomposes the shape of digits and background feature, so it is able to generalize to a new domain where digits are rotated with much difference compared with training.

\begin{table}[!t]
    \centering
    \caption{Classification accuracy on the PACS dataset with ResNet-18. 
    }
    \label{table:pacs}
	\begin{adjustbox}{max width=0.55\textwidth}
	\begin{threeparttable}
        \begin{tabular}{lcccc|c}
        \toprule
        \toprule
         Model  &A  &C  &S  &P  &Average\\
        \midrule
        ERM~\cite{arjovsky2019invariant}    &77.85 	   &74.86    &67.74   &95.73  &79.05\\
        IRM~\cite{arjovsky2019invariant}\tnote{*}    &70.31 	       &73.12    &75.51   &84.73  &75.92\\
        REx~\cite{krueger2020outofdistribution}\tnote{*} &76.22&73.76&66.00&95.21&77.80\\
        JiGen~\cite{carlucci2019domain}   &  79.42 & 75.25 &71.35& \textbf{96.03} &  80.51\\
        Mixup~\cite{zhang2017mixup}\tnote{*}   & 82.01  & 72.58  & 72.48& 93.29  & 80.09  \\
        CuMix~\cite{mancini2020towards}  &  \textbf{82.30}	 & 76.50 & 	72.60 & 95.10 &  81.60\\
        MLDG~\cite{li2018learning}      &79.50         &77.30    &71.50   &94.30  &80.70\\
        MMD~\cite{li2018domain}
        \tnote{*}	&	79.34&	73.76&	72.61 &94.19&	79.97\\
        \midrule
        \emph{DecAug}   & \emph{79.00}	 & \textbf{\emph{79.61}} & 75.64 & \emph{95.33} &  \textbf{\emph{82.39}}\\
        \bottomrule
        \bottomrule
        \end{tabular}
        \begin{tablenotes}
		    \item[*] Implemented by ourselves.
		\end{tablenotes}
    \end{threeparttable}
    \end{adjustbox}
    \vspace{-0.3cm}
\end{table}

In the PACS dataset, DecAug achieves the \textbf{\emph{state-of-the-art (SOTA)}} performance followed by Cumix and MASF when using ResNet-18 as the backbone network. The detail of our results on the PACS dataset is shown in Table~\ref{table:pacs}. For risk regularization methods, such as IRM and REx, they have degenerated performance in PACS compared to two variants of  the  MNIST dataset. This is because strong regularization terms are added in ERM to eliminate all unstable features across different environments. This technique can work well in the standard and “clean" dataset—-MNIST, where shapes of digits are relatively stable. However, in realistic scenarios, the shape of target objects can vary, so in different training environments, features for predicting object category can be unstable.

Different from PACS, whose domains are divided according to style, VLCS has only one specific style: Photo. 
In VLCS, the casual feature shape is relatively stable, this explains why risk regularization methods, such as IRM and REx, can achieve better results than ERM. As shown in Table~\ref{table:vlcs}, DecAug achieves the \textbf{\emph{state-of-the-art (SOTA)}} performance followed by CuMix. Besides, many DG methods achieve better results than ERM, which matches our expectation.

\begin{table}[!t]
    \centering
    \caption[Classification accuracy on the VLCS dataset with ResNet-18.]{Classification accuracy on the VLCS dataset with ResNet-18. All methods are implemented by ourselves.}
    \label{table:vlcs}
    \begin{adjustbox}{max width=0.55\textwidth}
        \begin{threeparttable}
            \begin{tabular}{lcccc|c}
                \toprule
                \toprule
                Model & V & L & C & S & Average\\
                \midrule
                ERM~\cite{arjovsky2019invariant}	   &72.85	&60.97	&97.87	&67.51	&74.80 \\
                IRM~\cite{arjovsky2019invariant}	&74.43	&64.74	&97.40	&67.20	&75.94\\
                REx~\cite{krueger2020outofdistribution}	&\textbf{76.11}	&62.60	&97.87	&65.48	&75.51\\
                Mixup~\cite{zhang2017mixup}	&71.86	&64.74	&98.11	&68.02	&75.68 \\
                CuMix~\cite{mancini2020towards}  &73.34&	66.49&	97.40&	69.74&	76.74 \\
                MLDG~\cite{li2018learning} &70.87&61.73	&91.03&\textbf{73.40}&	74.25\\
                MMD~\cite{li2018domain}
                &71.86	&63.11	&97.40	&69.64	&75.50\\
                DRO~\cite{sagawa2019distributionally}	&71.07	&62.73	&\textbf{99.52}	&70.25	&75.89\\
                \midrule
                \emph{DecAug}	&74.23 &\textbf{\emph{69.00}}	&98.58	&69.54	&\textbf{\emph{77.84}} \\
                \bottomrule
                \bottomrule
            \end{tabular}%
            \begin{tablenotes}
                \item
            \end{tablenotes}
        \end{threeparttable}
    \end{adjustbox}
    \vspace{-0.3cm}
\end{table}%

The recently proposed NICO dataset take more realistic generalization scenarios into consideration where foreground objects and backgrounds, \ie, contexts in the dataset have great difference. For example, there are pictures where a dog is on the grass in the training dataset, while a dog is on the beachside in the test dataset. 
We implemented the algorithms listed in Table~\ref{table:nico}. As shown, the proposed DecAug achieved the best generalization performances on two sets, with 85.23\% on animal and 80.12\% on vehicle, followed by JiGen. CuMix achieved 76.78\% (animal) and 74.74\% (vehicle) accuracy on NICO, indicating that mixing up (interpolating) data may fail to correct the spurious correlation between irrelevant features such as the background to the predicted category. The poor performance of DANN and IRM on NICO may probably due to the diversity shift.
Experiments on the NICO dataset further demonstrate the superiority of the proposed algorithmic framework. \emph{Our method has achieved the SOTA performance simultaneously on various OoD generalization tasks, indicating a new promising direction for OoD learning algorithm research}.

Two important features are taken into consideration in our work, which is the causal feature and correlated feature. The former is we want to focus on, because only causal feature determines the category of an object. The latter feature is what we want to ignore. In our experiment, IRM has good performance among three datasets: Colored MNIST, Rotated MNIST and VLCS. JiGen has good performance on Rotated MNIST, PACS, VLCS and NICO dataset simultaneously, with the worst performance in Colored MNIST(worse than the ERM). We speculate the performance of an algorithm varies among different datasets is due to different characteristics of datasets. Labels of Colored MNIST have a strong correlation with colors, the remaining datasets, however, do not have this characteristic. JiGen does not eliminate the correlation between label and colors in Colored MNIST, and in the meanwhile, the causal feature, shape, is randomly changed by this algorithm, which may result in a worse result than the baseline ERM method.

\begin{table}[!t]
    \centering
    \caption[Results of DecAug compared with different existing methods on NICO dataset.]{Results of DecAug compared with different methods on the NICO dataset. All methods are implemented by ourselves.}
    \label{table:nico}
	\begin{adjustbox}{max width=1.0\textwidth}
	\begin{threeparttable}
        \begin{tabular}{lcc|c}
        \toprule
        \toprule
        Model & Animal &  Vehicle&Average \\
        \midrule
        ERM~\cite{arjovsky2019invariant}	&75.87&	74.52&75.19\\
        IRM~\cite{arjovsky2019invariant}	& 69.63 & 64.94&67.28	 \\
        REx~\cite{krueger2020outofdistribution} &74.31& 66.10 &70.21 \\
        Mixup~\cite{zhang2017mixup}	 & 80.27&77.00&78.63	 \\
        Cumix~\cite{mancini2020towards}   &76.78& 74.74&75.76\\
        MTL~\cite{blanchard2017domain}  & 78.89&	75.11&77.00	 \\
        MMD~\cite{li2018domain}
        &70.91 & 68.04&69.47\\
        DRO~\cite{sagawa2019distributionally}  &77.61&74.59&76.10 		 \\
        CNBB~\cite{he2020towards}   & 78.16& 77.39&77.77\\
        \midrule
        \emph{DecAug} & \emph{\textbf{85.23}} & \emph{\textbf{80.12}}&\emph{\textbf{82.67}}\\
        \bottomrule
        \bottomrule
        \end{tabular}
     \end{threeparttable}
     \end{adjustbox}
     \vspace{-0.3cm}
\end{table}

\begin{figure*}[ht]
    \centering
    \includegraphics[height = 10.1cm]{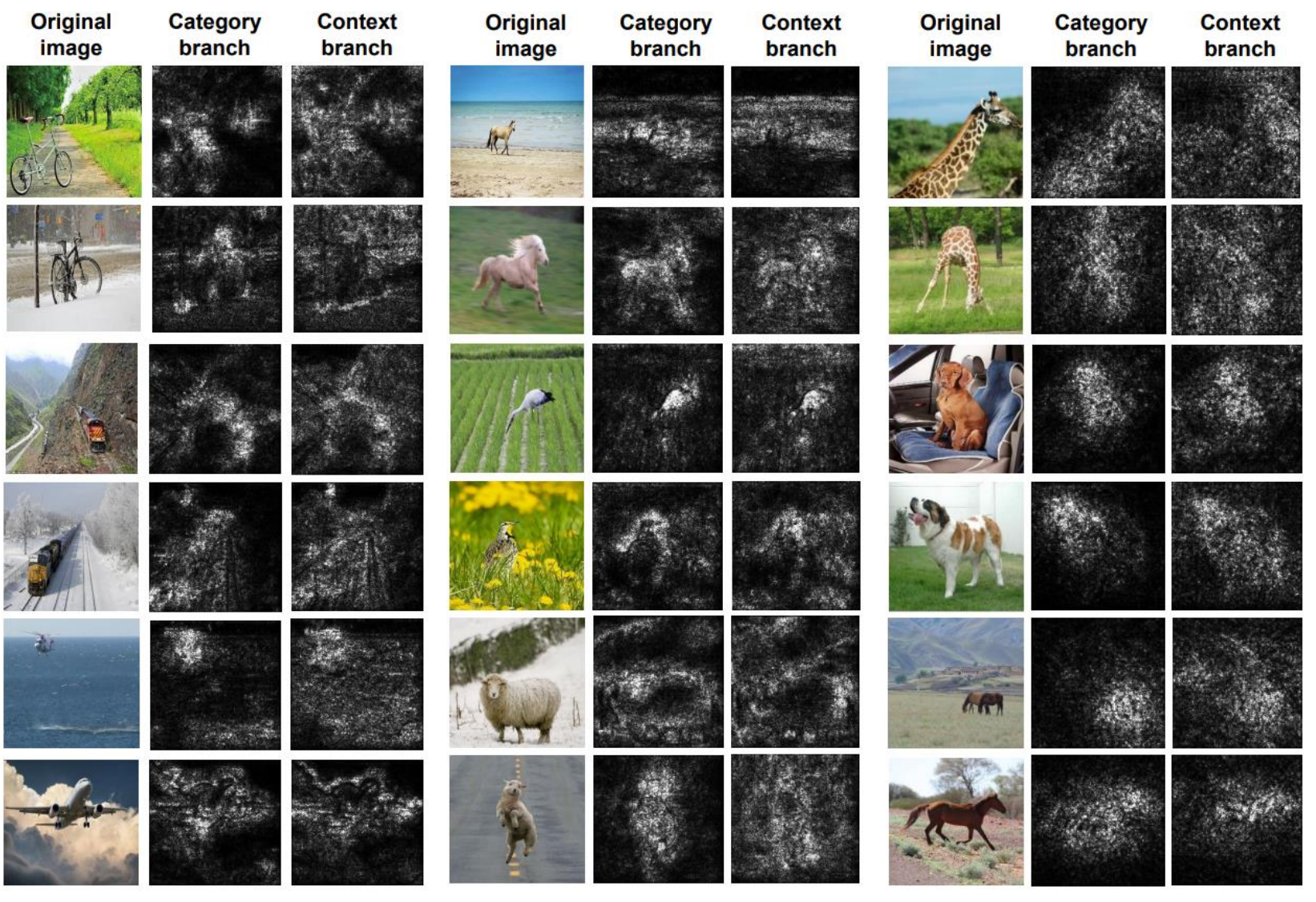}
    \caption[The gradient visualization of the decomposed category-related and context-related features.]{The gradient visualization of the decomposed category-related and context-related high-dimensional features. The first row is the original input images, the second row is its corresponding back propagation of the category branch and the last row is the back propagation of the context branch.}
    \label{fig:visualize}
\end{figure*}

\begin{table}[t]
    \centering
    \caption{Ablation study on the PACS dataset with ResNet-18.}
    \label{table:pacs_ablation}
	\begin{adjustbox}{max width=0.55\textwidth}
        \begin{tabular}{lcccc|c}
        \toprule
        \toprule
         Model &  A &   C &  S  &  P  &  Average\\
        \midrule
        DecAug (orth 0)  &  78.42	 & 78.32 & 	72.13 & 94.19 &  80.77\\
        DecAug (orth 0.0005) & 77.49 & 77.43 & 74.32 & 94.07& 80.76\\
        DecAug (orth 0.001) & 78.12 & 77.34 & \textbf{76.97} & 94.91 & 81.83\\
        DecAug (orth 0.01) &  \emph{\textbf{79.00}} & \emph{\textbf{79.61}} & 	\emph{75.64} & 	\emph{\textbf{95.33}} &  \emph{\textbf{82.39}}\\
        \bottomrule
        \bottomrule
        \end{tabular}
    \end{adjustbox}
    \vspace{-0.3cm}
\end{table}

\subsection{Ablation Studies and Sensitivity Analysis}
In this section, ablation studies and sensitivity analysis are presented to evaluate the effectiveness of \name{}.
Firstly, we test the effects of orthogonal loss. Then we tried different variants of \name{}, which shows the two branch architecture is needed and the gradient-based augmentation is better to perform on context branch. We also compared our method with vanilla multi-task learning. Furthermore, we use interpretability methods to understand the decomposed high-level representations.

The results are shown in Table~\ref{table:pacs_ablation}. It can be seen that without the orthogonal loss, our method achieves an average accuracy of 80.77\% that is higher than most of the methods in Table~\ref{table:pacs}. This is because the category and context losses also play the role of feature decomposition. The additional orthogonal loss enforces the gradients of the category and context losses to be orthogonal to each other, which helps to further decompose the features. As expected, with the increase of the orthogonal regularization coefficient $\lambda^{\text{orth}}$, the performance of DecAug can be improved. The experimental results confirm the effectiveness of the proposed orthogonal loss.

\begin{table}[!t]
    \centering
    \caption{Results of DecAug variants on the PACS dataset.}
    \label{table:orthfeatures}
	\begin{adjustbox}{max width=1.0\textwidth}
	\begin{threeparttable}
        \begin{tabular}{lc}
        \toprule
        \toprule
        Model & Average  \\
        \midrule
        Multi-task learning & 77.06 \\
        DecAug (DANN loss) & 81.00  \\
        DecAug (adversarial loss) & 76.37 \\
        DecAug (orth between features) & 79.90  \\
        DecAug (gradient-based orth)& \emph{\textbf{82.39}}  \\
        \bottomrule
        \bottomrule
        \end{tabular}
     \end{threeparttable}
     \end{adjustbox}
     \vspace{-0.3cm}
\end{table}

Following the same leave-one-domain-out validation protocol as in the previous works on PACS, we evaluate the proposed DecAug without the concatenate operation. 
The balance weight for the category branch and the context branch is one. Note that we also involve the orthogonal constraints in this experiment.
As can be seen in Table~\ref{table:concat_ablation}, the application of features concatenation achieves the best performance.

\begin{table}[!t]
    \centering
    \caption{Ablation study for features concatenate on PACS with ResNet-18.}
    \label{table:concat_ablation}
	\begin{adjustbox}{max width=0.8\textwidth}
        \begin{tabular}{lcccc|c}
        \toprule
        \toprule
         PACS &  Art Painting &   Cartoon &  Sketch  &  Photo  &  Average\\
        \midrule
        DecAug with concat & \textbf{79.00} & \textbf{79.61} & \textbf{75.64} & \textbf{95.33} & \textbf{82.39} \\
        DecAug without concat & 75.83 & 76.24 & 75.57 & 94.79 & 80.61 \\
        \bottomrule
        \bottomrule
        \end{tabular}
    \end{adjustbox}
\end{table}

We changed current orthogonal regularization to orthogonal constraints between features, refer to Table~\ref{table:orthfeatures}, which reaches 79.90\% on PACS, lower than the original DecAug. We also tried confusion regularization, as discussed in recent literature Open Compound Domain Adaptation~\cite{liu2020open}. It seems natural to incorporate confusion regularization into our method for better decomposition. However, after many trials, no improvements were observed.

We tried DecAug with DANN adversarial loss "orth" on PACS. As shown in Table~\ref{table:orthfeatures}, the result is around 81\%, lower than the original. This shows both gradient orthogonalization and semantic augmentation are indispensable parts of the algorithm. We tried "adversarial augmentation" to Jigsaw, the result is much lower than Jigsaw. This shows that the two branch architecture is needed and adversarial augmentation is better performed on the context predicting branch to improve OoD generalization via challenging neural networks to unseen context information.

In addition, we tried to add adversarial loss on the context branch. The resulting accuracy on PACS is 76.37\%, far lower than DecAug (82.41\%), indicating that adversarial loss may not able to discourages the representation from being predictive of the context labels under OoD settings.\\

The setting of OoD generalization is different from multi-task learning, in that OoD generalization focus on overcoming the spurious correlation and generalizing to unseen domains, whereas multi-task learning aims to achieve the Pareto-optimality across different tasks. Besides, a key difference is that DecAug applies gradient orthogonalization on the semantic feature z instead of network parameters $\theta$ that leads to Pareto-optimality in multi-task setting whereas, in our case, this disentangles the semantic features for category prediction and context prediction. We also compared DecAug with a vanilla multi-task learning, where two classifiers are used to classify categorical and context branches, separately. The resulting test accuracy in PACS is 77.06\%, which is much lower than DecAug (82.41\%). \\

\begin{figure}[t]
    \centering
    \includegraphics[height = 8.5cm]{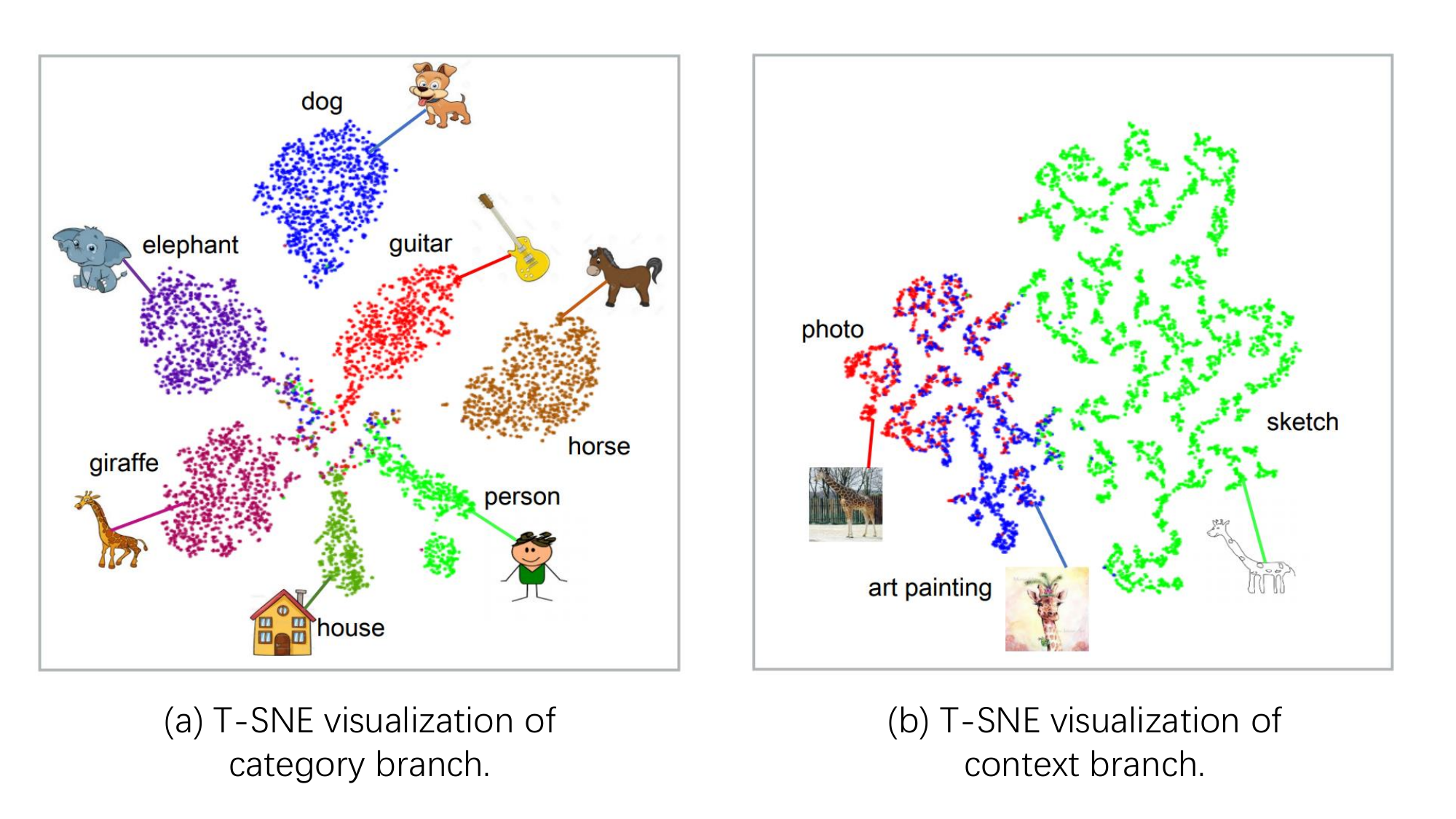}
    \caption[t-SNE visualization of the decomposed high-dimensional features.]{The t-SNE visualization of the decomposed high-dimensional category-related and context-related features. (a) Embedding of category branch versus category labels. (b) Embedding of context branch versus context labels. The difference between (a) and (b) shows the high-level category-related and context-related features are well decomposed.}
    \label{fig:t-SNE}
\end{figure}

We use deep neural network interpretability methods in \cite{Bengio2018_sanity} to explain the neural network's classification decisions as shown in Figure~\ref{fig:visualize}. It can be seen that the saliency maps of the category branch focus more on foreground objects, while the saliency maps of the context branch are also sensitive to background contexts that contain domain information. 

We further conducted t-SNE analysis of the representations of context branch versus the categorical labels. The t-SNE analysis showed that different classes in the categorical branch mix together and cannot tell the class clusters, indicating that the context branch consists of little information of the categorical labels (as shown in Figure~\ref{fig:t-SNE}).
This shows that our method well decomposes the high-level representations into two features that contain category and context information respectively. Later, by performing semantic augmentation on context-related features, our model breaks the inherent relationship between contexts and category labels and generalizes to unseen combinations of foregrounds and backgrounds. 
From Figure~\ref{fig:visualize}, we can see that the category branch learns more about the causal features for predicting categories, and the context branch learns more about the background. This further shows the effectiveness of our method to decompose images.

\chapter{NAS-OoD: Neural Architecture Search for Out-of-Distribution
Generalization}

\section{Introduction} \label{sec:intro}

\begin{figure}
    \centering
    \includegraphics[width=0.73\linewidth]{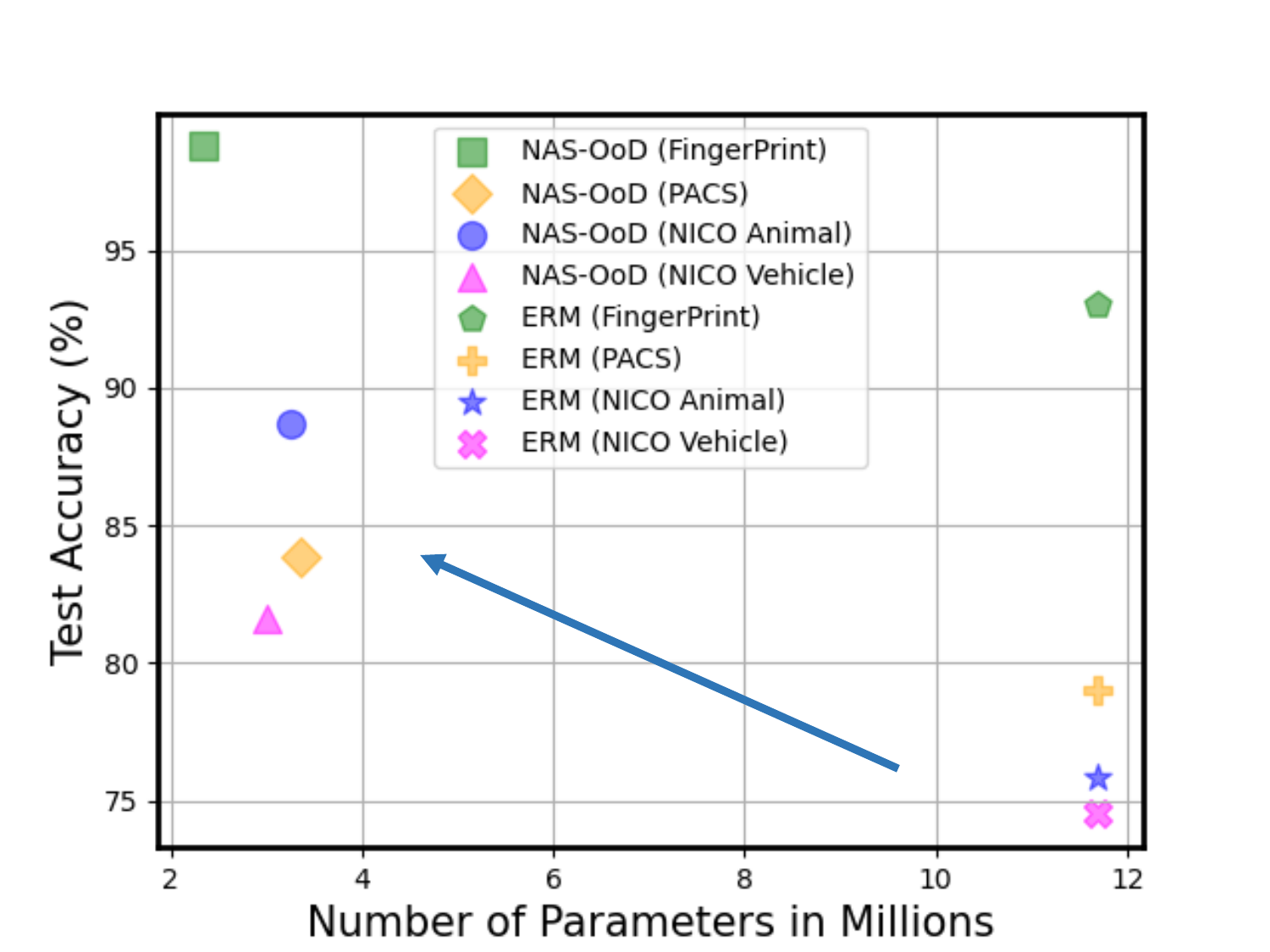}
    \caption{NAS-OoD performs significantly better than existing OoD generalization baselines in terms of test accuracy and network parameter numbers. The upper left points are better than lower right ones
    because they have higher test accuracy and lower parameter numbers.}
    \label{fig:testacc}
\end{figure}

Deep learning models have encountered significant performance drop in Out-of-Distribution (OoD) scenarios~\cite{bahng2019rebias,krueger2020outofdistribution},
where test data come from a distribution different from that of the training data. 
With their growing use in real-world applications in which mismatches of test and training data distributions are often observed~\cite{koh2020wilds}, extensive efforts have been devoted to improving generalization ability
~\cite{Li2017,arjovsky2019invariant,he2020towards,bai2020decaug}. Risk regularization methods~\cite{arjovsky2019invariant,ahuja2020invariant,xie2020risk}
aim to learn invariant representations across different training environments by imposing different invariant risk regularization. Domain generalization methods~\cite{Li2017,li2017learning,carlucci2019domain,zhou2020learning} learn models from multiple domains such that they can generalize well to unseen domains. Stable learning~\cite{Kuang2018,kuang2020stable,he2020towards} focuses on identifying stable and causal features for predictions.
Existing works, however, seldom consider the effects of architectures on generalization ability. 
On the other hand, some pioneer works suggest that different architectures show varying OoD generalization abilities~\cite{hendrycks2020pretrained, dapello2020simulating, li2020network}.
How a network's architecture affects its ability to handle OoD distribution shifts is still an open problem.

Conventional Neural Architecture Search (NAS) methods search for architectures with maximal predictive performance on the validation data that are randomly divided from the training data~\cite{zoph2017neural,pham2018efficient,liu2018darts,yang2020ista}. The discovered architectures are supposed to perform well on unseen test data under the assumption that data are Independent and Identically Distributed (IID).
While novel architectures discovered by recent NAS methods have demonstrated superior performance on different tasks with the IID assumption~\cite{tan2019efficientnet,ghiasi2019fpn,yao2019sm,hu2020count}, they may suffer from over-fitting in OoD scenarios, where the test data come from another distribution. 
A proper validation set that can evaluate the performance of architectures on the test data with distribution shifts is crucial in OoD scenarios.

In this paper, we propose robust NAS for OoD generalization (NAS-OoD) that searches architectures with maximal predictive performance on OoD examples generated by a conditional generator. An overview of the proposed method is illustrated in Figure~\ref{fig:framework}.
To do NAS and train an OoD model simultaneously, we follow the line of gradient-based methods for NAS~\cite{liu2018darts,xie2018snas,cai2018proxylessnas,hu2020dsnas,yang2020ista}, however, we extend that on several fronts. The discrete selection of architectures is relaxed to be differentiable by building all candidate architectures into a supernet with parameter sharing and adopting a softmax choice over all possible network operations. The goal for architecture search is to find the optimal architecture parameters that minimize the validation loss under the condition that the corresponding network parameters minimize the training loss. 

Instead of using part of the training set as the validation set, we train a conditional generator to map the original training data to synthetic OoD examples as the validation data. The parameters of the generator are updated to maximize the validation loss computed by the supernet. This update encourages the generator to synthesize data having a different distribution from the original training data since the supernet is optimized to minimize the error on the training data. To search for the architectures with optimal OoD generalization ability, the architecture parameters are optimized to minimize the loss on the validation set containing synthetic OoD data. This minimax training process effectively drives both the generator and architecture search to improve their performance and finally derive the robust architectures that perform well for OoD generalization.

Our main contributions can be summarized as follows:
\begin{enumerate}
    \item To the best of our knowledge, NAS-OoD is the first attempt to introduce NAS for OoD generalization, where a conditional generator is jointly optimized to synthesize OoD examples helping to correct the supervisory signal for architecture search.
    \item NAS-OoD gets the optimal architecture and all optimized parameters in a single run. The minimax training process effectively discovers robust architectures that generalize well for different distribution shifts.
    \item We take the first step to understanding the OoD generalization of neural network architectures systematically. We provide a statistical analysis of the searched architectures and our preliminary practice shows that architecture does influence OoD robustness.
    \item Extensive experimental results show that NAS-OoD outperforms the previous SOTA methods 
    and achieves the best overall OoD generalization performance on various types of OoD tasks with the discovered architectures having a much fewer number of parameters.
\end{enumerate}

\section{Related Work} \label{sec:rela}

\subsection{Out-of-Distribution Generalization}
\label{rela:ood}
Data distribution mismatches between training and testing set exist in many real-world scenes. Different methods have been developed to tackle OoD shifts. IRM~\cite{arjovsky2019invariant} targets to extract invariant representation from different environments via an invariant risk regularization. IRM-Games~\cite{ahuja2020invariant} aims to achieve the Nash equilibrium among multiple environments to find invariants based on ensemble methods. REx~\cite{krueger2020outofdistribution} proposes a min-max procedure to deal with the worst linear combination of risks across different environments. MASF~\cite{dou2019domain} adopts a framework to learn invariant features among domains. JiGen~\cite{carlucci2019domain} jointly classifies objects and solves unsupervised jigsaw tasks. CuMix~\cite{mancini2020towards} aims to recognize unseen categories in unseen domains through a curriculum procedure to mix up data and labels from different domains. DecAug~\cite{bai2020decaug} proposes a decomposed feature representation and semantic augmentation approach to address diversity and distribution shifts jointly. The work of \cite{hendrycks2019using} finds that using pre-training can improve model robustness and uncertainty. However, existing OoD generalization approaches seldom consider the effects of architecture which leads to suboptimal performances. In this work, we propose NAS-OoD, a robust network architecture search method for OoD generalization.

\subsection{Neural Architecture Search}
EfficientNet~\cite{tan2019efficientnet} proposes a new scaling method that uniformly scales all dimensions of depth, width, and resolution via an effective compound coefficient. EfficientNet design a new baseline which achieves much better accuracy and efficiency than previous ConvNets. One-shot NAS~\cite{bender2018understanding} discusses the weight sharing scheme for one-shot architecture search and shows that it is possible to identify promising architectures without either hypernetworks or RL efficiently. DARTS~\cite{liu2018darts} presents a differentiable manner to deal with the scalability challenge of architecture search. ISTA-NAS~\cite{yang2020ista} formulates neural architecture search as a sparse coding problem. In this way, the network in search satisfies the sparsity constraint at each update and is efficient to train. SNAS~\cite{xie2018snas} reformulates NAS as an optimization problem on parameters of a joint distribution for the search space in a cell. DSNAS~\cite{hu2020dsnas} proposes an efficient NAS framework that simultaneously optimizes architecture and parameters with a low-biased Monte Carlo estimate. NASDA~\cite{li2020network} leverages a principle framework that uses differentiable neural architecture search to derive optimal network architecture for domain adaptation tasks.
NADS~\cite{ardywibowo2020nads} learns a posterior distribution on the architecture search space to enable uncertainty quantification for better OoD detection and aims to spot anomalous samples.
The work \cite{chen2020robustness} uses a robust loss to mitigate the performance degradation under symmetric label noise. However, 
NAS overfits easily, the work \cite{yang2019evaluation, guo2020single} points out that NAS evaluation is frustratingly hard. Thus, it is highly non-trivial to extend existing NAS algorithms to the OoD setting.

\subsection{Robustness from Architecture Perspective}
\label{rela:arch}
Recent studies show that different architectures present different generalization abilities.
The work of \cite{zhang2021can} uses a functional modular probing method to analyze deep model structures under the OoD setting.
The work \cite{hendrycks2020pretrained} examines and shows that pre-trained transformers achieve not only high accuracy on in-distribution examples but also improvement of out-of-distribution robustness. The work \cite{dapello2020simulating} presents CNN models with neural hidden layers that better simulate the primary visual cortex improve robustness against image perturbations. The work \cite{dosovitskiy2020image} uses a pure transformer applied directly to sequences of image patches, which performs quite well on image classification tasks compared with relying on CNNs. The work of \cite{dong2020adversarially} targets to improve the adversarial robustness of the network with NAS and achieves superior performance under various attacks. However, they do not consider OoD generalization from the architecture perspective.

\section{Methodology}

\begin{figure*}
    \centering
    \includegraphics[width=1.0\linewidth]{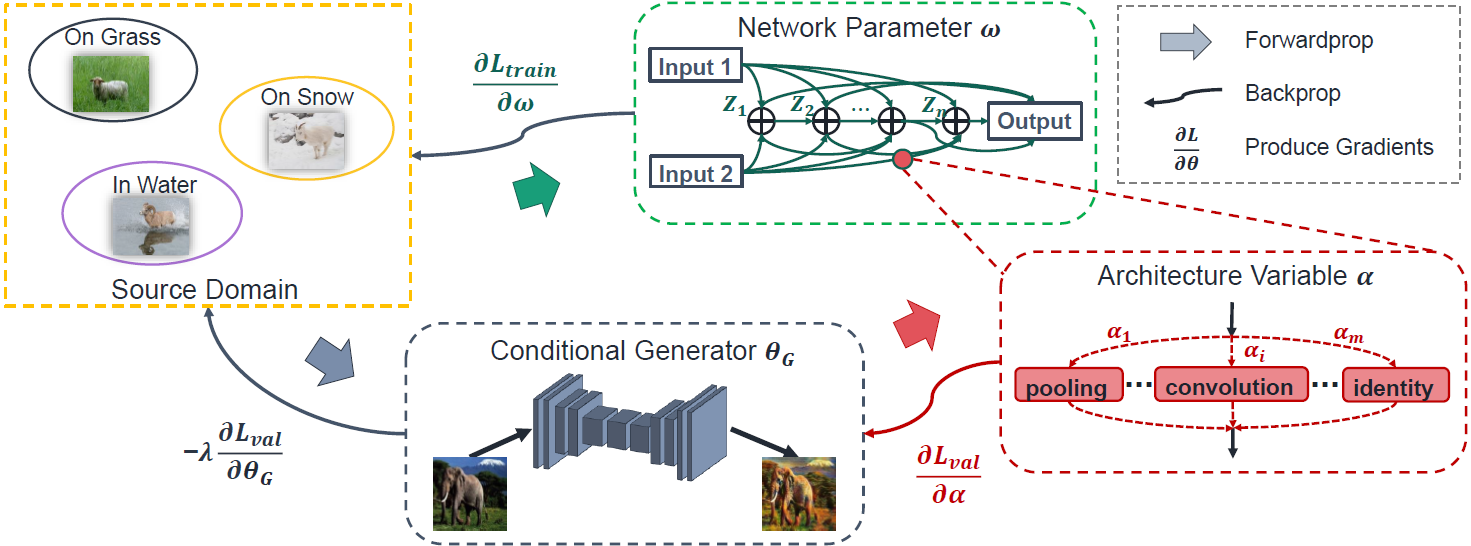}
    \caption
    {An overview of the proposed NAS-OoD. A conditional generator is learned to map the original training data to synthetic OoD examples by maximizing their losses computed by different neural architectures. Meanwhile, the architecture search process is optimized by minimizing the synthetic OoD data losses. 
    }
    \label{fig:nas-ood}
\end{figure*}

In this section, we first introduce preliminaries on conventional NAS and their limitations in OoD scenarios (Section~\ref{meth:nas}). Then, we describe the details of our robust Neural Architecture Search for OoD generalization (Section~\ref{meth:nas-ood}). 

\subsection{Preliminaries on Differentiable Neural Architecture Search}
\label{meth:nas}

Conventional NAS methods mainly search for computation cells as the building units to construct a network~\cite{liu2018darts, xie2018snas, hu2020dsnas}. The search space of a cell is defined as a directed acyclic graph with $n$ ordered nodes $\{z_1, z_2, ..., z_n\}$ and edges $\mathcal{\xi}=\{e^{i,j}|1 \le i < j \le n\}$. Each edge includes $m$ candidate network operations chosen from a pre-defined operation space $\mathcal{O} = \{o_1, o_2, ..., o_m\}$, such as max-pooling, identity and dilated convolution. The binary variable $s_{k}^{(i,j)} \in \{0, 1\}$ denotes the corresponding active connection. Thus, the node can be formed as:
\begin{equation}
z_j = \sum_{i=1}^{j-1}\sum_{k=1}^{m}s_k^{(i,j)}o_k(z_i) = \boldsymbol{s}_{j}^{T}\boldsymbol{o}_j,
\end{equation}
where $\boldsymbol{s}_{j}$ is the vector consists of $s_k^{(i,j)}$ and $\boldsymbol{o}_j$ denotes the vector formed by $o_k(z_i)$. As the binary architecture variables $s_j$ is hard to optimize in a differentiable way, recent DARTS-based NAS methods make use of the continuous relaxation in the form of
\begin{equation}
s_k^{(i,j)}=\text{exp}(\alpha_k^{(i,j)})/\sum_k\text{exp}(\alpha_k^{(i,j)}),
\end{equation}
and optimize $\alpha_k^{(i,j)}$ as trainable architecture parameters~\cite{liu2018darts}, which can be formulated as the following bilevel optimization problem:

\begin{equation}
\begin{aligned}
&\boldsymbol{\alpha}^* = \mathop{\arg \min}_{\boldsymbol{\alpha}} \ \ell_{\text{val}}(\boldsymbol{\omega}^{\ast}(\boldsymbol{\alpha}), \boldsymbol{\alpha}), \\
&\mathop{s.t.}\quad \boldsymbol{\omega}^{\ast}(\boldsymbol{\alpha}) = \mathop{\arg \min}_{\boldsymbol{\omega}}\ \ell_{\text{train}}(\boldsymbol{\omega}, \boldsymbol{\alpha}),
\end{aligned}
\end{equation}

where $\boldsymbol{\alpha}$ denotes the architecture variable vector formed by $\alpha_k^{(i,j)}$, and $\boldsymbol{\omega}$ denotes the supernet parameters. 
$\ell_{\text{train}}$ and $\ell_{\text{val}}$ denote the training and validation losses, respectively. In the search phase, $\boldsymbol{\alpha}$ and $\boldsymbol{\omega}$ are optimized in an alternate manner.

The validation data used for the above architecture search method are usually divided from the training data. Previous research demonstrates that the derived architectures perform well on different tasks~\cite{liu2018darts, xie2018snas, hu2020dsnas} when the training and test data are IID. 
However, when dealing with OoD tasks, where the test data come from another distribution, using part of the training set as the validation set may cause NAS methods suffer from over-fitting and the searched architectures to be sub-optimal in this situation. Thus, a proper validation set is needed to effectively evaluate the performance of discovered architectures on the test set in OoD scenarios.


\subsection{NAS-OoD: Neural Architecture Search for OoD Generalization}
\label{meth:nas-ood}

In OoD learning tasks, we are provided with $K$ source domains. The target is to discover the optimal network architecture that can generalize well to the unseen target domain. 
In the following descriptions, let $\boldsymbol{\alpha}$, $\boldsymbol{\omega}$ and $\boldsymbol{\theta}_{G}$ denote the parameters for architecture topology, the supernet and the conditional generator $G(\cdot,\cdot)$, respectively. The conditional generator $G(\cdot,\cdot)$ takes data $\boldsymbol{x}$ and domain labels $\boldsymbol{\tilde{k}}$ as the input. Let $\ell_{\text{train}}$ be the training loss function, and $\ell_{\text{val}}$ be the validation loss function.

To generate a proper validation set for OoD generalization in NAS, as shown in Figure~\ref{fig:nas-ood}, a conditional generator is learned to generate novel domain data by maximizing the losses on different neural architectures, while the optimal architecture variables are optimized by minimizing the losses on generated OoD images. This can be formulated as a constrained minimax optimization problem as follows:
\begin{equation}\label{eq:minmax}
\centering
\begin{aligned}
    &\mathop{\min}_{\boldsymbol{\alpha}} \mathop{\max}_{G}  \ell_{\text{val}}(\boldsymbol{\omega}^{\ast}(\boldsymbol{\alpha}), \boldsymbol{\alpha}, G(\boldsymbol{x},\boldsymbol{\tilde{k}})),\\ 
    &\text{s.t.} \quad
    \boldsymbol{\omega}^{\ast}(\boldsymbol{\alpha}) = \mathop{\arg \min}_{\boldsymbol{\omega}} \ell_{\text{train}}(\boldsymbol{\omega}, \boldsymbol{\alpha}, \boldsymbol{x}),
\end{aligned}
\end{equation}

where $G(\boldsymbol{x},\boldsymbol{\tilde{k}})$ is the generated data from the original input data $\boldsymbol{x}$ on the novel domain $\boldsymbol{\tilde{k}}$. This is different from NAS methods' formulation as we introduce a generator to adversarially generate challenging data from original input for validation loss to search for network architectures. This can avoid over-fitting problem by using the same data for optimizing neural network parameters and architectures as shown in our experiment. Solving this problem directly will involve calculating second-order derivatives that will bring much computational overhead and the constraint is hard to realize. Thus, we introduce the following practical implementations of our algorithm.

\begin{algorithm}[t]\small
\caption{NAS-OoD: Neural Architecture Search for OoD generalization}
\label{alg:nas-ood}
\begin{algorithmic}[1]
\REQUIRE Training set $\mathcal{D}$, batch size $n$, learning rate $\mu$. 
\ENSURE $\boldsymbol{\alpha}, \boldsymbol{\omega}, \boldsymbol{\theta}_G$.
\STATE Initialize $\boldsymbol{\alpha}$, $\boldsymbol{\omega}$, $\boldsymbol{\theta}_G$; 
\REPEAT
\STATE Sample a mini-batch of training images $\{(x_i, y_i)\}_{i=1}^{n}$;
\STATE Generate novel domain data: $x_i^{\text{syn}} \leftarrow \text{G} (x_i, \boldsymbol{\tilde{k}})$;
\STATE $\boldsymbol{\theta}_G \leftarrow \boldsymbol{\theta}_G - \mu \cdot \nabla_{\boldsymbol{\theta}_G} \ell_{\text{aux}}$ according to Eqn.~\eqref{eq:G};
\STATE $\boldsymbol{\omega} \leftarrow \boldsymbol{\omega} - \mu \cdot \nabla_{\boldsymbol{\omega}} \ell_{\text{train}}(\boldsymbol{\omega}, \boldsymbol{\alpha}, x_i)$ according to Eqn.~\eqref{eq:iterative};
\STATE $\boldsymbol{\theta}_G \leftarrow \boldsymbol{\theta}_G + \mu \cdot \nabla_{\boldsymbol{\theta}_G} \ell_{\text{val}}(\boldsymbol{\omega}, \boldsymbol{\alpha},x_i^{\text{syn}})$ according to Eqn.~\eqref{eq:iterative};
\STATE $\boldsymbol{\alpha} \leftarrow \boldsymbol{\alpha} - \mu \cdot \nabla_{\boldsymbol{\alpha}} \ell_{\text{val}}(\boldsymbol{\omega}, \boldsymbol{\alpha},x_i^{\text{syn}})$ according to Eqn.~\eqref{eq:iterative};
\UNTIL convergence;
\end{algorithmic}
\end{algorithm}

Inspired by the previous work in meta-learning~\cite{finn2017model}, we approximate the multi-step optimization with the one-step gradient when calculating the gradient for $\boldsymbol{\alpha}$. Different source domains are mixed together in architecture search, while the domain labels are embedded in the generator auxiliary loss training process which will be explained later. For the architecture search training process, architecture parameters $\boldsymbol{\alpha}$, network parameters $\boldsymbol{\omega}$ and parameters for conditional generator $\boldsymbol{\theta}_G$ are updated in an iterative training process: 
\begin{equation}\label{eq:iterative}
\begin{aligned}
&\boldsymbol{\omega} \leftarrow \boldsymbol{\omega} - \mu \cdot \nabla_{\boldsymbol{\omega}} \ell_{\text{train}}(\boldsymbol{\omega}, \boldsymbol{\alpha}, \boldsymbol{x}),\\
&\boldsymbol{\theta}_G \leftarrow \boldsymbol{\theta}_G + \mu \cdot \nabla_{\boldsymbol{\theta}_G} \ell_{\text{val}}(\boldsymbol{\omega}, \boldsymbol{\alpha}, G(\boldsymbol{x}, \boldsymbol{\tilde{k}})),\\
&\boldsymbol{\alpha} \leftarrow \boldsymbol{\alpha} - \mu \cdot \nabla_{\boldsymbol{\alpha}} \ell_{\text{val}}(\boldsymbol{\omega}, \boldsymbol{\alpha}, G(\boldsymbol{x},\boldsymbol{\tilde{k}})).
\end{aligned}
\end{equation}

To train the generator and improve consistency, we apply an additional cycle consistency constraint to the generator:
\begin{equation}\label{eq:cycle}
\ell_{\text{cycle}} = ||\text{G}(\text{G}(\boldsymbol{x}_{k}, \widetilde{k}), k)-\boldsymbol{x}_{k}||_{1},
\end{equation}
where $\boldsymbol{x}_{k}$ denotes data from $K$ source domains with domain $\{s_1, s_2, ..., s_K\}$, $\widetilde{k}$ denotes the domain index for the generated novel domain $s_{K+1}$, and $||\cdot||_1$ refers to L1 norm. This can regularize the generator to be able to produce data from and back to the source domains.

To preserve semantic information, we also require the generated data $\text{G}(\boldsymbol{x}_{k}, \widetilde{k})$ to keep the same category as the original data $\boldsymbol{x}_{k}$.
\begin{equation}\label{eq:ce}
\ell_{\text{ce}}  = \text{CE}(Y(\text{G}(\boldsymbol{x}_{k}, \widetilde{k})),Y^{\ast}(\boldsymbol{x}_{k})),
\end{equation}
where $\text{CE}$ be the cross-entropy loss, $Y$ is a classifier with a few convolutional layers pretrained on training data, $Y^{\ast}(\cdot)$ is the ground-truth labels for the input data.

The total auxiliary loss for generator is defined as follows:
\begin{equation}\label{eq:G}
\ell_{\text{aux}} = \lambda_{cycle} \cdot \ell_{\text{cycle}} + \lambda_{ce} \cdot \ell_{\text{ce}},
\end{equation}
where $\lambda_{ce}$ and $\lambda_{cycle}$ 
are hyper-parameters.

Compared with the gradient-based perturbation~\cite{shankar2018generalizing}, the conditional generator is able to model a more sophisticated distribution shift due to its intrinsic learnable nature.
The NAS-OoD algorithm is outlined in Algorithm ~\ref{alg:nas-ood}.

\section{Illustrative Results}

In this section, we conduct numerical experiments to evaluate the effectiveness of our proposed NAS-OoD method. 
To provide a comprehensive comparison with baselines, We compare our proposed NAS-OoD with the SOTA algorithms from various OoD areas, including empirical risk minimization (ERM~\cite{arjovsky2019invariant}),
invariant risk minimization (IRM~\cite{arjovsky2019invariant}),
risk extrapolation (REx~\cite{krueger2020outofdistribution}),
domain generalization by solving jigsaw puzzles (JiGen~\cite{carlucci2019domain}),
mixup (Mixup~\cite{zhang2017mixup}),
curriculum mixup (Cumix~\cite{mancini2020towards}),
marginal transfer learning (MTL~\cite{blanchard2017domain}),
domain adversarial training (DANN~\cite{ganin2016domain}), 
correlation alignment (CORAL~\cite{sun2016deep}),
maximum mean discrepancy (MMD~\cite{li2018domain}),
distributionally robust neural network (DRO~\cite{sagawa2019distributionally}),
convnets with batch balancing (CNBB~\cite{he2020towards}),
cross-gradient training (CrossGrad~\cite{shankar2018generalizing}), and the recently proposed
decomposed feature representation and semantic augmentation (DecAug~\cite{bai2020decaug}), etc. 
More details about the baseline methods are shown in the Appendix.

For ablation studies, We also compare NAS-OoD with SOTA NAS methods, such as differentiable architecture search (DARTS~\cite{liu2018darts}),
stochastic neural architecture search (SNAS~\cite{xie2018snas}),
efficient and consistent neural architecture search by sparse coding (ISTA-NAS~\cite{yang2020ista}). This is to test whether naively combining NAS methods with OoD learning algorithms can improve the generalization performance.

Our framework was implemented with PyTorch 1.4.0 and CUDA v9.0. We conducted experiments on NVIDIA Tesla V100. 
Following the design of~\cite{choi2018stargan}, our generator model has an encoder-decoder structure, which consists of two down-sampling convolution layers with stride 2, three residual blocks, and two transposed convolution layers with stride 2 for up-sampling. The domain indicator is encoded as a one-hot vector. The one-hot vector is first spatially expanded and then concatenated with the input image to train the generator. More implementation details can be found in the Appendix.

\subsection{Evaluation Datasets}
\label{exp:dataset}

We evaluate our NAS-OoD on four challenging OoD datasets: NICO Animal, NICO Vehicle, PACS, and Out-of-Distribution Fingerprint (OoD-FP) dataset where methods have to be able to perform well on data distributions different from training data distributions. The evaluation metric is the classification accuracy of the test set. The number of neural network parameters is used to measure the computational complexity for comparison between different neural network architectures.

NICO (Non-I.I.D. Image with Contexts) datasets consists of two datasets, i.e., NICO Animal with 10 classes and NICO Vehicle with 9 classes. 
The NICO dataset is a recently proposed OoD generalization dataset in the real scenarios~\cite{he2020towards}, (see Figure~\ref{fig:imgnico}), 
which contains different contexts, such as different object poses, positions, and backgrounds across the training, validation, and test sets.

PACS (Photo, Art painting, Cartoon, Sketch) dataset is commonly used in OoD generalization. (see Figure~\ref{fig:imgpacs}).
It contains four domains with different image styles, namely photo, art painting, cartoon, and sketch with seven categories (dog, elephant, giraffe, guitar, horse, house, person). We follow the same leave-one-domain-out validation experimental protocol as in~\cite{Li2017}, i.e., we select three domains for training and the remaining domain for testing for each time.

\begin{figure}[t]
    \centering
\includegraphics[width=0.7\linewidth]{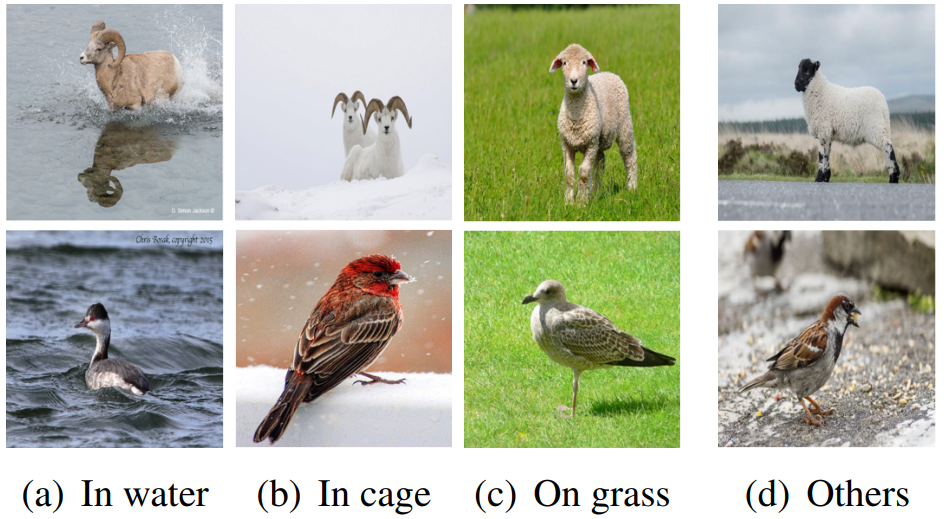}
    \caption{Examples of the out-of-distribution data from the NICO dataset with contexts (a) In water, (b) On snow, (c) On grass, and(d) Others.}
    \label{fig:imgnico}
\end{figure}

\begin{figure}[t]
    \centering
    \includegraphics[width=0.7\linewidth]{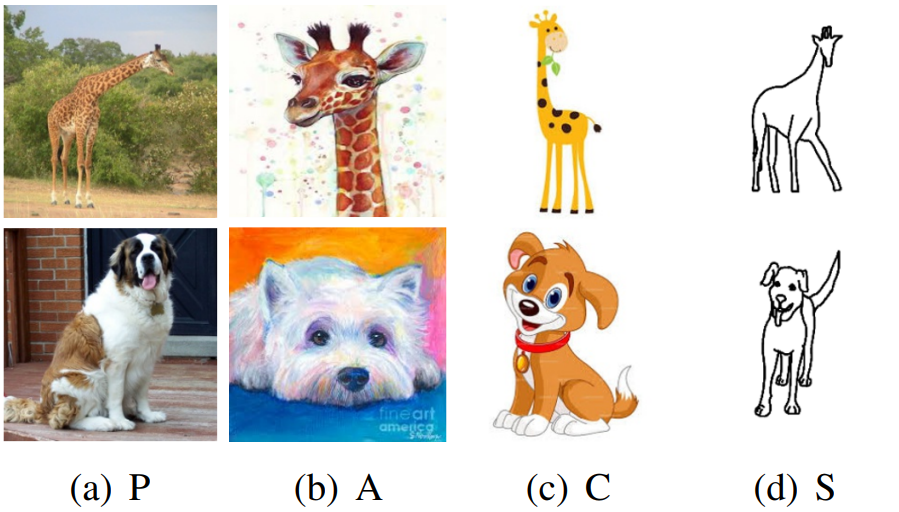}
    \caption{Typical examples of out-of-distribution data with diversity shift from the PACS dataset. (a) Photo. (b) Art Painting. (c) Cartoon. (d) Sketch. 
    }
\label{fig:imgpacs}
\end{figure}

\begin{figure}[t]
    \centering
    \includegraphics[width=0.7\linewidth]{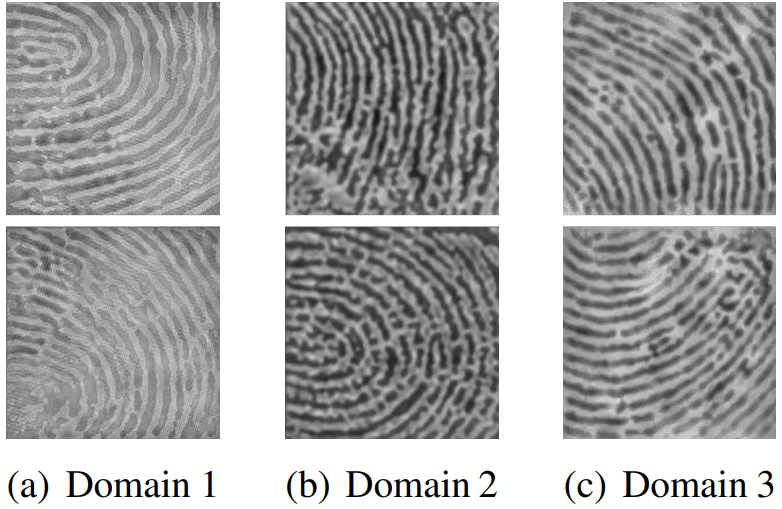}
    \caption{Typical examples of out-of-distribution data in the OoD-FP dataset with three different domains.}
    \label{fig:imgFP}
\end{figure}

OoD-FP (Out-of-Distribution Fingerprint) dataset is a real industry dataset that contains three domains corresponding to different fingerprint collection devices on different brands of mobile phones. (see Figure~\ref{fig:imgFP}). 
In the fingerprint recognition task, the goal is to learn to distinguish whether input fingerprints are from the users' fingerprints stored in the dataset. Due to the hardware implementation differences, fingerprints exhibit different styles from different devices. In our setting, the goal is to learn a universal fingerprint recognition neural network to generalize on the fingerprints collected from unseen datasets.

\begin{table}[t]
\centering
\caption{Results of NAS-OoD compared with different methods with ResNet-18 (11.7M) on the NICO dataset.}
\label{table:nico-nas-ood}
\begin{adjustbox}{max width=0.95\textwidth}
\begin{threeparttable}
\begin{tabular}{lcc|c}
\toprule
\toprule
Model & Animal & Vehicle & Average \\
\midrule
ERM~\cite{arjovsky2019invariant}	               &75.87   &74.52   &75.19  \\
IRM~\cite{arjovsky2019invariant}	               &59.17   &62.00   &60.58	 \\
REx~\cite{krueger2020outofdistribution}            &74.31   &66.20   &70.25  \\
JiGen~\cite{carlucci2019domain}                    &84.95   &79.45   &82.20  \\
Mixup~\cite{zhang2017mixup}\tnote{*}	           &80.27   &77.00   &78.63	 \\
Cumix~\cite{mancini2020towards}                    &76.78   &74.74   &75.76  \\
MTL~\cite{blanchard2017domain}\tnote{*}            &78.89   &75.11   &77.00	 \\
MMD~\cite{li2018domain}\tnote{*}                   &70.91   &68.04   &69.47  \\
CNBB~\cite{he2020towards}                          &78.16   &77.39   &77.77  \\
DecAug~\cite{bai2020decaug}                        &85.23   &80.12   &82.67  \\
\midrule
\emph{NAS-OoD} &\emph{\textbf{88.72}} &\emph{\textbf{81.59}} &\emph{\textbf{85.16}}\\
\midrule
Params (M) & 3.25 & 3.00 & 3.13
\\
\bottomrule
\bottomrule
\end{tabular}
\begin{tablenotes}
	\item[*] Implemented by ourselves.
\end{tablenotes}
\end{threeparttable}
\end{adjustbox}
\vspace{-10pt}
\end{table}

\subsection{Results and Discussion}
\label{exp:discuss}
NAS-OoD achieves the SOTA performance \emph{simultaneously} on various datasets from different OoD research areas, such as stable learning, domain generalization, and real industry dataset.

\begin{table}[t]
\centering
\caption{Classification accuracy on the PACS dataset compared with different methods with ResNet-18 (11.7M).}
\label{table:pacs_pretrain}
\begin{adjustbox}{max width=0.48\textwidth}
\begin{threeparttable}
\begin{tabular}{lcccc|c}
\toprule
\toprule
Model  &A  &C  &S  &P  &Average\\
\midrule
ERM~\cite{arjovsky2019invariant}              &77.85   &74.86   &67.74   &95.73   &79.05  \\
IRM~\cite{arjovsky2019invariant}              &70.31   &73.12   &75.51   &84.73   &75.92  \\
REx~\cite{krueger2020outofdistribution}       &76.22   &73.76   &66.00   &95.21   &77.80  \\
JiGen~\cite{carlucci2019domain}               &79.42   &75.25   &71.35   &96.03   &80.51  \\
Mixup~\cite{zhang2017mixup}\tnote{*}          &82.01   &72.58   &72.48   &93.29   &80.09  \\
CORAL~\cite{sun2016deep}\tnote{*}             &80.49   &74.32   &75.06   &94.09   &80.99  \\
MMD~\cite{li2018domain}\tnote{*}	          &79.34   &73.76   &72.61   &94.19   &79.97  \\
L2A-OT~\cite{zhou2020learning} &83.30 &78.20 &73.60 &96.20 &82.80 \\
DecAug~\cite{bai2020decaug}                   &79.00   &79.61   &75.64   &95.33   &82.39  \\
\midrule
\emph{NAS-OoD} &\emph{\textbf{83.74}} &\emph{\textbf{79.69}} &\emph{\textbf{77.27}} &\emph{\textbf{96.23}} &\emph{\textbf{84.23}} \\
\midrule
Params (M) & 3.51 & 3.44 & 3.35 & 3.15 & 3.36 \\
\bottomrule
\bottomrule
\end{tabular}
\begin{tablenotes}
	\item[*] Implemented by ourselves.
\end{tablenotes}
\end{threeparttable}
\end{adjustbox}
\vspace{-10pt}
\end{table}

The results for the challenging NICO dataset are shown in Table~\ref{table:nico-nas-ood}. From Table~\ref{table:nico-nas-ood}, the proposed NAS-OoD method achieves the SOTA performance simultaneously on the two subsets of the NICO dataset with a much fewer number of parameters. Specifically, NAS-OoD achieves $88.72\%$ on NICO Animal and $81.59\%$ on NICO Vehicle with only around $3.1$ million parameters compared with DecAug achieving $82.67\%$ accuracy but with $11.7$ million parameters. The superior performance of NAS-OoD also confirms the possibility of improving the neural network's OoD generalization performance by searching for the architecture, which provides an orthogonal way to improve the OoD generalization.

We also compare our methods with different domain generalization methods on the PACS dataset. The results are shown in Table~\ref{table:pacs_pretrain}. Similarly, we observe that NAS-OoD achieves SOTA performance on all the four domains and the best average generalization performance of $83.89\%$ with only $3.36$ million of network parameters. The generalization accuracy is much better than previous OoD algorithms DecAug ($82.39\%$), JiGen ($80.51\%$), IRM ($75.92\%$) with ResNet-18 backbone, which are the best OoD approaches before NAS-OoD. The network parameters for ResNet-18 is $11.7$ million which is much larger than the network searched by our NAS-OoD. Note that the relative performance for some algorithms may change drastically between NICO and PACS datasets whereas the proposed NAS-OoD algorithm can generalize well simultaneously on datasets from different OoD research areas.

To test the generalization performance of NAS-OoD on real industry datasets, we compare NAS-OoD with other methods on OoD-FP dataset. The results are shown in Table~\ref{table:finger}. NAS-OoD consistently achieves good generalization performance with the non-trivial improvement compared with other methods. NAS-OoD achieves a $1.23\%$ error rate in the fingerprint classification task which almost reduces the error rate by around $70\%$ compared with the second-best method--MMD. This demonstrates the superiority of NAS-OoD and especially its potential to be practically useful in real industrial applications.

\begin{table}[t]
\centering
\caption[Classification accuracy compared to different methods with ResNet-18 backbone (11.7M) on the OoD-FP dataset.]{Classification accuracy compared to different methods with ResNet-18 backbone (11.7M) on the OoD-FP dataset. All methods are implemented by ourselves.
}
\label{table:finger}
\begin{adjustbox}{max width=0.48\textwidth}
\begin{tabular}{lccc|c}
\toprule
\toprule
Model  & Domain 1  & Domain 2  & Domain 3   & Average\\
\midrule
ERM~\cite{arjovsky2019invariant}        &93.75  &92.70  &92.70  &93.05  \\
IRM~\cite{arjovsky2019invariant}        &95.83  &87.50  &84.37  &89.23  \\
REx~\cite{krueger2020outofdistribution} &97.91  &91.66	&92.70  &94.09  \\
Mixup~\cite{zhang2017mixup}             &96.87  &97.91  &90.62	&95.13  \\
MTL~\cite{blanchard2017domain}          &95.83	&97.91  &90.62  &94.78  \\
DANN~\cite{ganin2016domain}             &95.83	&97.91	&86.45  &93.39  \\
CORAL~\cite{sun2016deep}                &94.79  &97.91  &91.66	&94.78  \\
MMD~\cite{li2018domain}              	&96.87  &95.83  &94.79	&95.83  \\
\midrule
\emph{NAS-OoD} &\emph{\textbf{99.27}}  &\emph{\textbf{99.49}} &\emph{\textbf{97.54}} &\emph{\textbf{98.77}} 
\\
\midrule
Params (M) & 2.28 & 2.28 & 2.43 & 2.33
\\
\bottomrule
\bottomrule
\end{tabular}
\end{adjustbox}
\vspace{-10pt}
\end{table}

\begin{table}[t]
\centering
\caption[Results of NAS-OoD compared with other NAS methods.]{Results of NAS-OoD compared with other NAS methods. The baselines are implemented by ourselves.}
\label{table:niconas}
\begin{adjustbox}{max width=0.95\textwidth}
\begin{threeparttable}
\begin{tabular}{lcc|c}
\toprule
\toprule
Model &Animal &Vehicle &Average \\
\midrule
DARTS~\cite{liu2018darts}  & 83.67  &75.55 & 79.61 \\
DARTS + IRM~\cite{liu2018darts}	& 82.29 & 72.24 & 77.26 \\
SNAS~\cite{xie2018snas}  & 85.96& 80.04 & 83.00 \\
SNAS + IRM~\cite{xie2018snas}	 & 82.94 & 76.51 & 79.73 \\
ISTA-NAS~\cite{yang2020ista} & 86.70 & 80.56 & 83.63\\
ISTA-NAS + IRM~\cite{yang2020ista} & 82.57 & 77.61 & 80.09 \\
\midrule
\emph{NAS-OoD} & \emph{\textbf{88.72}} & \emph{\textbf{81.59}}&\emph{\textbf{85.16}}\\
\bottomrule
\bottomrule
\end{tabular}
\end{threeparttable}
\end{adjustbox}
\end{table}

\begin{table}[t]
\centering
\caption{NAS-OoD variants.}
\label{table:variants}
\begin{adjustbox}{max width=0.9\textwidth}
\begin{threeparttable}
\begin{tabular}{lcc|c}
\toprule
\toprule
Model &Animal &Vehicle &Average \\
\midrule
Random Sample &80.92 &76.43 &78.68 \\
NAS-OoD w/o cycle loss     &86.88 &80.85 &83.86 \\
\midrule
\emph{NAS-OoD} & \emph{\textbf{88.72}} & \emph{\textbf{81.59}}&\emph{\textbf{85.16}}\\
\bottomrule
\bottomrule
\end{tabular}
\end{threeparttable}
\end{adjustbox}
\end{table}

\subsection{Ablation Study}
\label{exp:abla}

In this section, we first test whether naively combining NAS methods with domain generalization methods can achieve good OoD generalization performances. We conduct experiments on the NICO dataset. 
The results are shown in Table~\ref{table:niconas}. It can be seen that using NAS methods only, such as DARTS, can achieve only $79.61\%$ average accuracy, significantly lower than most compared compositions. This is in stark contrast with the good performance for NAS methods on IID generalization tasks where training and test datasets have similar distributions. This is because NAS methods are doing variational optimization by finding not only the best parameters but also the best function for fitting whereas this can help NAS methods to achieve good performance in IID settings. In OoD settings, where test data distributions differ significantly from training data distributions, NAS methods can overfit the training data distribution and achieve sub-optimal performance. Besides, we can also observe that naively combining the NAS methods with OoD learning algorithms, such as IRM, brings no statistically significant performance gain. The reason is that many OoD learning algorithms are based on implicit or explicit regularization added to the ERM loss. NAS methods will explore the search space to fit the loss terms and the regularization term may be ignored as NAS methods may exploit only the ERM loss, thus bringing no performance gain. This also confirms that generating OoD data is needed during training to avoid over-fitting.

As shown in Table~\ref{table:variants}, the average results of randomly sampled architectures are 79.45\% (Animal) and 75.70\% (Vehicle), which is significantly lower than those of the architectures searched by NAS-OoD. We also conduct an ablation study on this auxiliary loss, the average accuracy on NICO without cycle loss is 83.86\%, which is low than our proposed method 85.16\%. This shows the effectiveness of the auxiliary cycle loss to facilitate the searching process.

\subsection{Analysis of Searched Architectures}
After setting up the NAS-OoD framework, we want to analyze whether any special patterns for searched cell-based network architectures and whether the NAS-OoD framework can stably find consistent architectures.

To check whether the found special pattern is consistent during the training process, we plot the operation type's percentage during the training process in Figure~\ref{fig:multiprocess}. In Figure~\ref{fig:multiprocess}, we can find that as the training proceeds, the architecture found by NAS-OoD is converging to the pattern that the percentage of dilated convolution $3\times3$ is higher and the separable convolution $3\times3$ is lower. This might be because dilated convolution has a larger receptive field than separable convolution which only receives one channel for each convolution kernel and a large receptive field can better learn the shape features of objects rather than the spurious features, such as color and texture.

\begin{figure}[!t]
    \centering
    \includegraphics[width=0.7\linewidth]{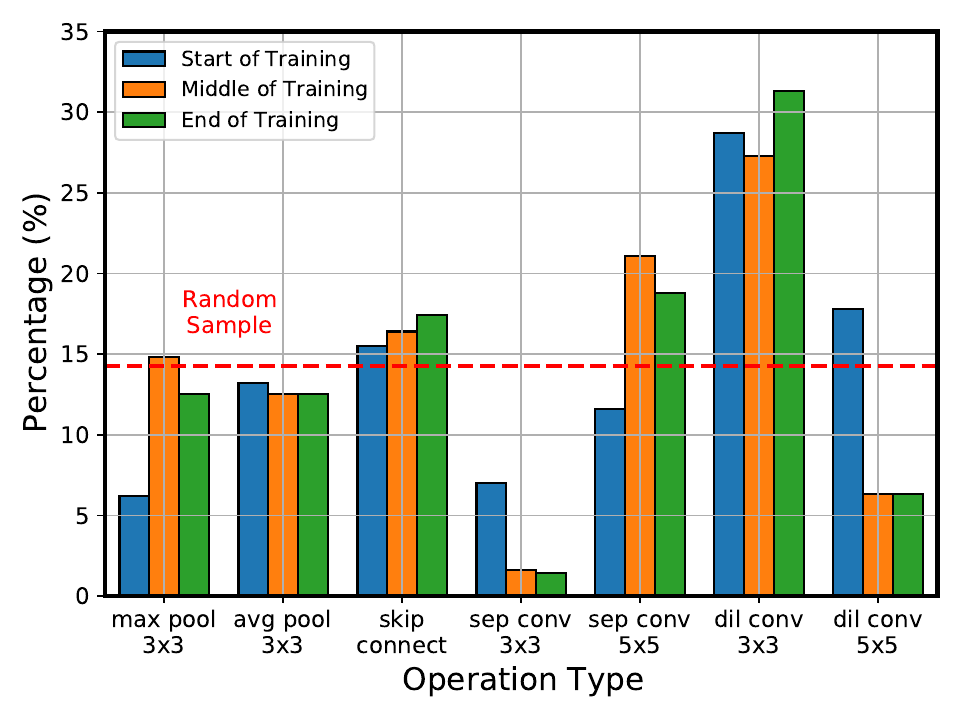}
    \caption[Temporal stability of search architecture.]{Temporal stability of search architecture.(Better viewed in the zoom-in mode)}
    \label{fig:multiprocess}
\end{figure}

\begin{figure}[!t]
    \centering
    \includegraphics[width=0.7\linewidth]{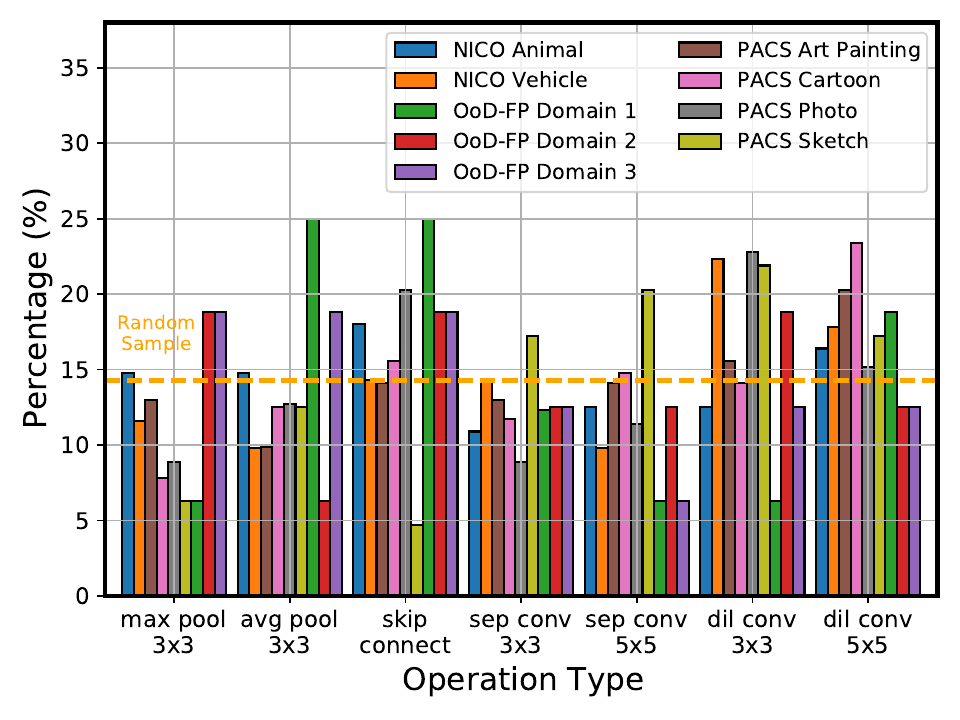}
    \caption[Statistical analysis of searched architectures on different datasets.]{Statistical analysis of searched architectures on different datasets.(Better viewed in the zoom-in mode)}
    \label{fig:multidatasets}
\end{figure}

To check whether the architecture patterns searched by NAS-OoD are similar across different datasets, we plot the operation type's percentage for different datasets in Figure~\ref{fig:multidatasets}.

We found there are similarities of architectures found on different datasets. Specifically, the searched NAS-OoD architectures tend to contain more convolutional operations with a large kernel size compared with randomly sampled architectures. This may be because a larger kernel size convolution operation has larger receptive fields, which makes better use of contextual information compared with a small kernel size. NAS-OoD architectures also present more skip connection operations compared with random selection and locate on both skip edges and direct edges, which can better leverage both low-level texture features and high-level semantic information for recognition. 
There is some previous study show that densely connected pattern benefits model robustness.
NAS-OoD searched for more dilated convolutions than normal convolutions, which may be due to that the dilated convolutions enlarge the receptive fields.

\begin{figure}[!t]
    \centering
    \includegraphics[width=0.47\linewidth]{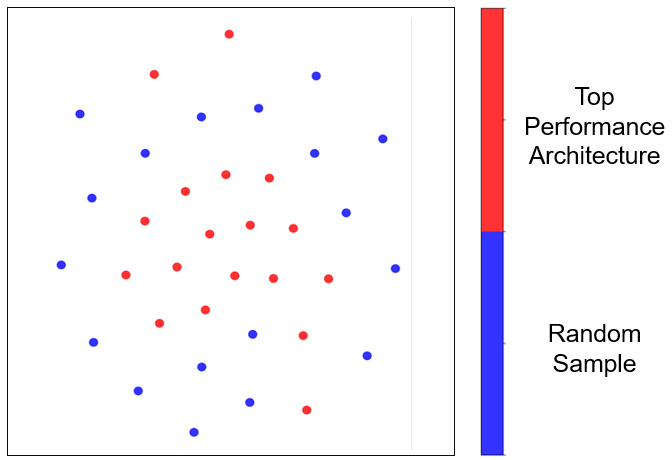}
    \caption{t-SNE visualization of architecture topology.}
    \label{fig:t-SNE-nas}
\end{figure}

We do t-SNE visualization to visualize the embedding of the topology of $\alpha$ in the two-dimensional space to see whether the top performance architecture has different patterns compared with average architectures. We randomly selected nine top performance architectures found by NAS-OoD and compared them with other randomly selected architectures. The results are shown in Figure~\ref{fig:t-SNE-nas}. The top performance architecture is significantly different from random architectures. This demonstrates that there are indeed special architecture patterns that can help the deep neural network to generalize well on OoD settings.

\begin{figure}[!t]
    \centering
    \includegraphics[height = 0.36\linewidth]{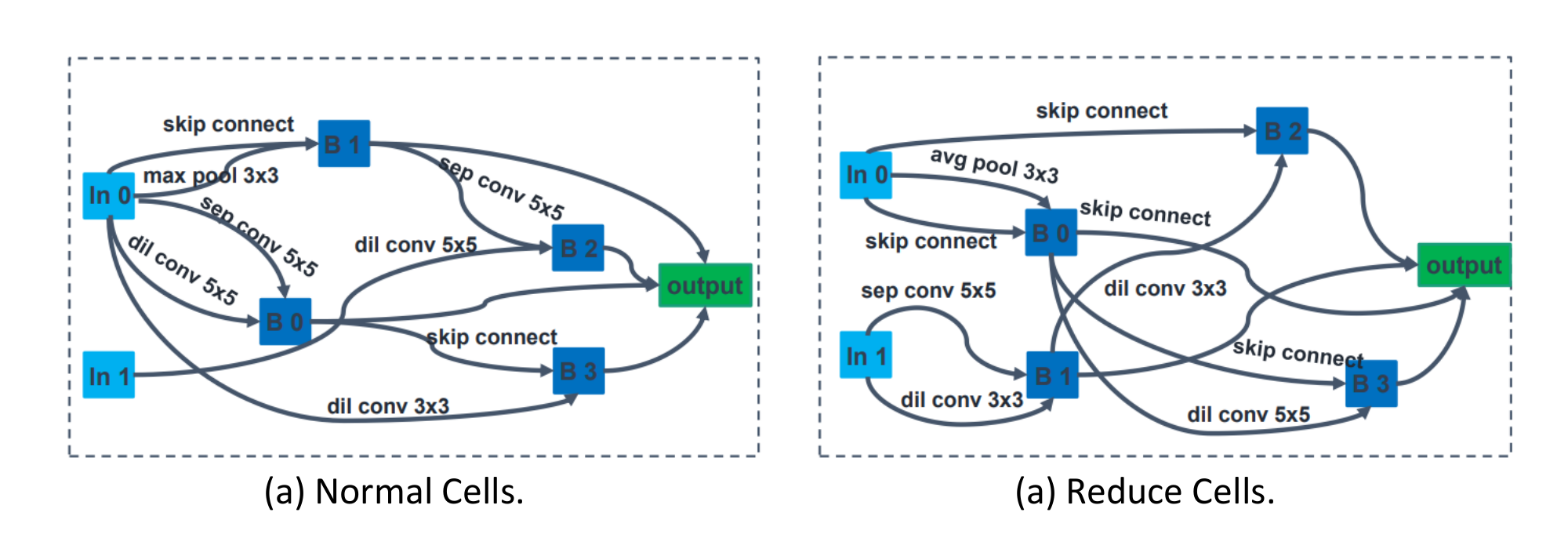}
    \caption[Typical examples of searched robust architectures.]{Typical examples of searched robust architectures on NICO dataset.(Better viewed in the zoom-in mode)}
    \label{fig:nicocell}
\end{figure}

\begin{figure}[!t]
    \centering
    \includegraphics[width=0.58\linewidth]{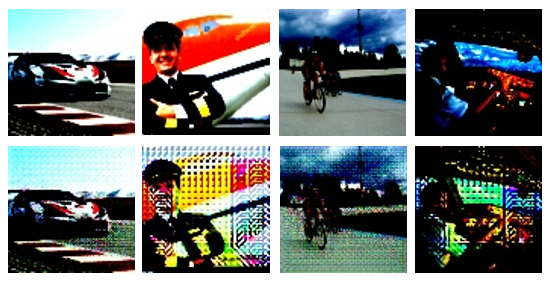}
    \caption{Some examples of synthetic images. The first row shows the original images, and the second row is its corresponding synthetic images.}
    \label{fig:syn}
\end{figure}

As illustrated in Figure~\ref{fig:nicocell}, we present the detailed structures of the best cells discovered on different datasets using NAS-OoD. Figure~\ref{fig:nicocell} (a) show the normal cells and (b) demonstrate the reduce cells. The searched cell contains two input nodes, four intermediate nodes, and one output node. Each intermediate node has two edges to the previous nodes which consist of both direct edge and skip edge, and the operation is presented on each edge. Besides, all the intermediate nodes are connected and aggregated to the output node.

We also visualize the generated OoD data. in Figure~\ref{fig:syn}. 
We find that the generated images show different properties and are clearly different from the source images. The conditional generator tends to generate images with different background patterns, textures, and colors. The semantic different make them helpful for improving out-of-distribution generalization ability.

\chapter{Conclusion}

In this thesis, we have provided a comprehensive and systematic framework to understand distribution shifts. We analyzed two-dimensional OoD shifts (i.e., correlation shift and diversity shift) from different perspectives: data augmentation and neural network architectures. We find that while existing OoD algorithms results are not consistent across datasets including ERM, learned data augmentation and searched robust architectures do better than baselines methods over various distribution shifts.

Firstly, we propose DecAug, a novel decomposed feature representation and semantic augmentation method for various OoD generalization tasks. High-level representations for the input data are decomposed into category-related and context-related features to deal with the diversity shift between training and test data. Gradient-based semantic augmentation is then performed on the context-related features to break the spurious correlation between context features and image categories. To the best of our knowledge, this is the first method that can simultaneously achieve the SOTA performance on various OoD generalization tasks from different research areas, indicating a new research direction for OoD generalization research. 

Then, we propose a robust neural architecture search framework that is based on differentiable NAS to understand the importance of network architecture against Out-of-Distribution robustness. We jointly optimize NAS and a conditional generator in an end-to-end manner. The generator is learned to synthesize OoD instances by maximizing their losses computed by different neural architectures, while the goal of the architecture search is to find the optimal architecture parameters that minimize the synthesized OoD data losses. Our study presents several valuable observations on designing robust network architectures for OoD generalization.

The research field of out-of-distribution generalization is still evolving. We discuss some possible further research directions below:
    Firstly, investigating relaxing the assumption for different environments of OoD generalization methods is a valuable direction remains to be done.
    In addition, improving the feature representation learning and disentanglement by leveraging the latent heterogeneity inside data that can perform simultaneously well under different types of distribution shifts.
    Last but not the least, investigating the robustness of the recent emergence of vision transformer model, and understand the role of self-attention mechanism is also a promising direction which has potential to strengthen the search space for NAS-OoD.

\newpage
\addcontentsline{toc}{chapter}{Reference}
\bibliographystyle{IEEEtranN}
\bibliography{main}

\end{document}

%% file: 1_title.tex
\thispagestyle{empty}
\null\vskip0.5in
\begin{center}
    {\Large \thesistitle}
  \vfill
  \vspace{20mm}

  by

  \vspace{4mm}

  \thesisauthor \\
  \vfill
  \vspace{20mm}

  A Thesis Submitted to\\
  The Hong Kong University of Science and Technology \\
  in Partial Fulfillment of the Requirements for\\
  the Degree of Master of Philosophy \\
  in \programname \\
  \vfill \vfill
  \thesisdate, Hong Kong
  \vfill
\end{center}

\vfill

%% file: 5_acknowledgement.tex
\centerline{{\bf \Large Acknowledgments}} \vspace{5mm} \noindent

First of all, I would like to thank my dear advisor
Professor S.-H. Gary Chan, for his years of mentoring, advice, and encouragement. 
It was my great luck and honor to work under his supervision. 
I have learned a lot from him about how to develop, evaluate, express, and defend my work.
I would have never gotten to where I am today without his support and continuous encouragement.

I would like to thank all my colleagues at the Hong Kong University of Science and Technology, where I enjoyed a wonderful life during my postgraduate study, including Jiajie Tan, Guanyao Li, Zhangyu Chang, Wong Wang Kit, Song Wen, and many others. I was fortunate to meet and learn a lot from all these people.

Next, I want to express my sincere gratitude to Dr. Nanyang Ye, Dr. Fengwei Zhou, Dr. Lanqing Hong, and Dr. Han-Jia Ye for their guidance and valuable suggestions for research projects and broaden my research horizons. They are all excellent researchers who I look up to very much. Besides, I would like to thank Dr. Zhenguo Li for providing me with the great opportunity to go out to experience the industry world.

Last but not least, I would like to express my deepest thanks to my parents for their years of company and support to pursue my career dream.